\documentclass[twoside,leqno,twocolumn]{article}
\usepackage{mathptmx}
\usepackage[noend]{algpseudocode}
\usepackage[bm]{pauli_new}

\usepackage{defines}

\usepackage{amsfonts}
\usepackage{graphicx}
\usepackage{booktabs}
\usepackage{cite}
\usepackage{fixltx2e}
\usepackage[hyphens]{url}
\usepackage{microtype}
\usepackage{xspace}
\usepackage{algorithm}
\usepackage{array}
\usepackage{color}
\usepackage{subfig}
\usepackage[misc]{ifsym}

\usepackage{mathtools}

\graphicspath{{figs/}}

\begin{document}

\title{\Large Latitude: A Model for Mixed Linear--Tropical Matrix Factorization
}
\author{Sanjar Karaev\thanks{\rule{0pt}{1.1em}Max Planck Institute for Informatics, Germany.} \\
\and
James Hook\thanks{University of Bath, United Kingdom. \newline
\hspace*{1.25em}\texttt{\{skaraev,pmiettin\}@mpi-inf.mpg.de}, \texttt{j.l.hook@bath.ac.uk}} \\
\and
Pauli Miettinen\footnotemark[1]}
\date{}

\maketitle







\begin{abstract} 
  Nonnegative matrix factorization (NMF) is one of the most frequently-used matrix factorization models in data analysis. A significant reason to the popularity of NMF is its interpretability and the `parts of whole' interpretation of its components. Recently, max-times, or subtropical, matrix factorization (SMF) has been introduced as an alternative model with equally interpretable `winner takes it all' interpretation. In this paper we propose a new mixed linear--tropical model, and a new algorithm, called Latitude, that combines NMF and SMF, being able to smoothly alternate between the two. In our model, the data is modeled using the latent factors and latent parameters that control whether the factors are interpreted as NMF or SMF features, or their mixtures. We present an algorithm for our novel matrix factorization. Our experiments show that our algorithm improves over both baselines, and can yield interpretable results that reveal more of the latent structure than either NMF or SMF alone.



\end{abstract}


\section{Introduction}
\label{sec:introduction}
Matrix factorizations are a popular method for extracting latent structure from the data. Different factorizations find different types of structure. For example, singular value decomposition (SVD) and principal component analysis (PCA) can be used to find the directions of the greatest variance in the data. In other cases, we might want to decompose the matrix into nonnegative components to gain so-called ``parts-of-whole'' interpretation. For that, we would use some nonnegative matrix factorization (NMF) algorithm. Or perhaps, instead of taking the sum of the nonnegative components, we are only interested on the largest elements to gain ``winner-takes-it-all'' interpretation; for that, we would use subtropical matrix factorizations (SMF).
Matrix factorizations are global models, meaning that they apply their structures, be that SVD, NMF, or something else, to the whole matrix. But it is not clear that any data is only a result of a single model. Indeed, it can be that parts of the data are formed using a sum of the rank-1 components, while another part is formed by taking the element-wise maximums. Consider, for example, the classical example of movie ratings data, that is, \by*{people}{movies} matrix containing the movies' ratings. It is often assumed that these ratings have a latent low-rank structure, that there exists some $k$ factors that dictate how much different people like different movies, and that the users' final rating is a linear combination of these factors. For example, Alice might like some movie because she likes the director, the lead actor, and the genre, though she's less keen on the supporting actress. But it is equally plausible that some factor is so dominant, that the rating is dictated by that factor alone. For example, Bob might like all Star Wars movies simply because they are Star Wars movies, and completely irrespective of their director, actors and actresses, or other factors. In this situation, taking the largest value instead of the sum would be a better model.
In this paper we present a \emph{mixed linear--tropical model} that allows us to mix NMF and subtropical matrix decompositions. This provides us much more accurate decompositions than what we can achieve using either NMF or SMF -- or even SVD -- alone (see Section~\ref{sec:experiments}). That we can improve over the base models, NMF and SMF, indicates that our hypothesis of the data being of mixed structure is correct.
In addition to giving a better reconstruction error, our model is also highly interpretable, and uncovers interesting novel structure from the data. Namely, we can study which elements are more NMF-style and which are more SMF-style.
Our algorithm for finding the mixed linear--tropical structure is called \Latitude, as it varies smoothly between the tropic (SMF) and pole (NMF). \Latitude can be used to decompose relatively large data sets, as it scales linearly with the input data.
\paragraph{Main contributions.}
In this paper we present a novel matrix factorization model, called mixed linear--tropical model (Section~\ref{sec:model}) and a scalable algorithm for finding a decomposition in this model (Section~\ref{sec:algorithms}). Our experiments (Section~\ref{sec:experiments}) show that our algorithm finds decompositions that have smaller reconstruction error than what NMF or SMF methods -- or even SVD -- can find, and that the results are also intuitive and reveal interesting structures from the data sets. 

\section{Notation and Basic Definitions}
\label{sec:notation}

In this section we present the basic notation and briefly recall NMF and SMF. We postpone the discussion of the related work to Section~\ref{sec:related-work}.

\paragraph{Basic notation.}

Throughout this paper we will denote the $i$th row of matrix $\mA$ with $\mA_i$ and the $j$th column with $\mA^j$. Element $(i, j)$ of $\mA$ is $\mA_{ij}$. We use $\rg$ to denote the nonnegative real numbers $[0, \infty)$, and $\N$ to denote the natural numbers $\{1, 2, \ldots\}$. The Frobenius norm of a matrix $\mA$ is denoted by $\norm{\mA}_F = \bigl(\sum_{ij}\mA_{ij}^2\bigr)^{1/2}$.

\paragraph{Nonnegative matrix factorization.}

In nonnegative matrix factorization (NMF), we are given a nonnegative matrix $\mA\in\rgnm$ and target rank $k$, and our task is to find nonnegative factor matrices $\mB\in\rgnk$ and $\mC\in\rgkm$ that minimize $\norm{\mA - \mB\mC}_F$. Alternatively, we can write $\mB\mC = \sum_{i=1}^k \mB^i\mC_i$, where each $\mB^i\mC_j$ is a nonnegative rank-1 component matrix. Each component matrix contributes a nonnegative part to the total sum, and it is standard to interpret these rank-1 components as ``parts of a whole.'' Over the years, many different algorithms have been proposed to solve NMF, with methods based on alternating least squares optimization or multiplicative update rules being the most prominent ones (see \cite{cichocki09nonnegative} for a comprehensive treatise).

\paragraph{Subtropical matrix factorization.}

Subtropical matrix factorization (SMF) is similar to NMF, but it replaces the sum with the maximum in the component formulation: $\mB\maxprod\mC = \max_{i=1}^k\{\mB^i\mC_i\}$, where the maximum is taken element-wise. Equivalently, SMF is a matrix factorization over the \emph{subtropical (or max-times) semi-ring}, that is, values $\rg$ endowed with the addition operation $\max$ and the multiplication operation $\times$ (i.e.\ the standard multiplication).  To recap, in SMF, we are given a nonnegative matrix $\mA\in\rgnm$ and target rank $k$, and our task is to find nonnegative factor matrices $\mB\in\rgnk$ and $\mC\in\rgkm$ that minimize $\norm{\mA - \mB\maxprod\mC}_F = \norm{\mA - \max_{i=1}^k\{\mB^i\mC_i\}}_F$.

Since only the element-wise largest element has effect to the final product, SMF is said to exhibit the ``winner-takes-it-all'' structure.
This tends to yield sparser factor matrices \cite{karaev16capricorn,karaev16cancer}. Note also that $(\mB\maxprod\mC)_{ij} \leq (\mB\mC)_{ij}$ for all $i$ and $j$. Since SMF is taken over the subtropical semiring, it is possible that the factorization obtains smaller reconstruction error than SVD.

It should be noted that the concept of a rank-1 matrix in NMF and SMF coincide, even though rank-$k$ decompositions are generally different. This is a key feature for our model. In \emph{tropical algebra} the summation operation is $\max$ and the multiplication is $+$. Hence, matrix $\mA$ has \emph{tropical rank-1} if there exists vectors $\va$ and $\vb$ such that $\mA_{ij} = \va_i + \vb_j$. For more on tropical algebra, see Section~\ref{sec:related-work} and references therein.


\section{The Mixed Linear--Tropical Model}
\label{sec:model}

Rather than describing the data using NMF or subtropical matrix factorization, we propose a hybrid model that incorporates them both and allows for a smooth transition between the two. Ideally, given an input matrix $\mA \in \rgnm$, we want to be able to determine what elements $\mA_{ij}$ are better represented using the standard algebra, and which ones require the subtropical one. Namely, we seek factor matrices $\mB \in \rgnk$ and $\mC \in \rgkm$ and parameters $\mAlpha \in \rnm$ such that
\begin{equation} \label{mix:prod:general}
\mA_{ij} \approx f(\mAlpha_{ij}) (\mB \maxprod \mC) + g(\mAlpha_{ij}) (\mB \mC)_{ij} \; ,
\end{equation}
for some functions $f$ and $g$ that we will define below.
By representing $\mA$ as a ``mixture'' of the normal and subtropical products of the factor matrices, we allow for more flexibility in fitting the data. Since by altering the parameter matrix $\mAlpha$ we can choose different mixing coefficients for different elements of $\mA$, it is possible to better explain data that has piecewise NMF and piecewise subtropical structure. Moreover, since the functions $f(\mAlpha_{ij})$ and $g(\mAlpha_{ij})$ don't have to be restricted to binary values, we can also express some elements $\mA_{ij}$ as a weighted sum of $(\mB \maxprod \mC)_{ij}$ and $(\mB \mC)_{ij}$.

It is important to note that the equation \eqref{mix:prod:general} is quite general, and unless we impose restrictions on functions $f$ and $g$, as well as the matrix $\mAlpha$,  our model will overfit the data. When it comes to choosing the proper functions $f$ and $g$, there is a trade-off between fitting the data and keeping the model simple. In this paper we use the convex combination $f(\mAlpha_{ij}) = \mAlpha_{ij}, g(\mAlpha_{ij}) = 1 - \mAlpha_{ij},  \mAlpha_{ij} \in [0, 1]$, which is very simple, and at the same time provides an intuitive transition from the standard product at $\mAlpha_{ij}=0$ to the subtropical product at $\mAlpha_{ij}=1$. We obtain
\begin{equation} \label{mix:prod:convex}
  \mA_{ij} \approx \mAlpha_{ij} (\mB \maxprod \mC) + (1 - \mAlpha_{ij}) (\mB \mC)_{ij}
\end{equation}
for $\mAlpha_{ij} \in [0, 1]$.

When choosing $\mAlpha$ we are faced with a similar trade-off. Indeed, if allow $\mAlpha$ be unconstrained, we can fit arbitrarily complex matrices with \emph{constant} factor matrices, as the following proposition illustrates.

\begin{proposition}
  \label{prop:constant_factors}
  Let $\mA \in [1, 2]^{n\times m}$ and let $k=4$. There exists $\mAlpha\in[0,1]^{n\times m}$, $\mB\in\rgnk$, and $\mC\in\rgkm$ such that all entries of $\mB$ and $\mC$ are the same and that $\mA_{ij} = \mAlpha_{ij} (\mB \maxprod \mC) + (1 - \mAlpha_{ij}) (\mB \mC)_{ij}$ for all $i=1,\ldots, n$ and $j=1,\ldots, m$.
\end{proposition}

\begin{proof}
  Let all entries of $\mB$ and $\mC$ be $\sqrt{3}/2$. Then for any $1 \le i \le n$ and $1 \le j \le m$ we have $(\mB \maxprod \mC)_{ij} = 3/4$ and $(\mB \mC)_{ij} = 3$ . Now if we set 
  \begin{equation} \label{alpha:perfect:fit}
    \mAlpha_{ij} = \frac{\mA_{ij} - (\mB \mC)_{ij}} {(\mB \maxprod \mC)_{ij} - (\mB \mC)_{ij}} \;,
  \end{equation}
  then $0 \le \mAlpha_{ij} \le 1$ holds. By plugging \eqref{alpha:perfect:fit} into \eqref{mix:prod:convex}, we obtain $\mAlpha_{ij} (\mB \maxprod \mC)_{ij} + (1 - \mAlpha_{ij}) (\mB \mC)_{ij}  = \mA_{ij}$, concluding the proof.
\end{proof}

Being able to decompose essentially arbitrary matrix into constant factor matrices shows that unrestricted $\mAlpha$ can have too much power. To constrain $\mAlpha$, we enforce it to have essentially a tropical rank-1 structure: 
  \begin{equation} \label{alpha:sigmoid}
    \mAlpha_{ij} = \sigma(\co_i + \ro_j) \;,
  \end{equation}
  where $\co \in \rgc$ and $\ro \in \rgr$ are arbitrary vectors, and $\sigma(x) = 1/(1 + e^{-x})$ is the sigmoid function. 

Now, given the factors $\mB \in \rgnk$, $\mC \in \rgkm$ and the parameter vectors $\co \in \rgc$, $\ro \in \rgr$, we can define their mixed linear--tropical product, $\defprod$, elementwise as follows
  \begin{equation} \label{mixprod}
    (\defprod)_{ij} =  \mAlpha_{ij} (\mB \maxprod \mC)_{ij} + (1 - \mAlpha_{ij}) (\mB \mC)_{ij} \;,
  \end{equation}
  where  $\mAlpha_{ij} = \sigma(\co_i + \ro_j)$.

  It is trivial to see that when elements in both $\co$ and $\ro$ tend to $-\infty$, we have $\mB \mixprod_{\co, \ro} \mC \to \mB \mC$. The greater the element $\mAlpha_{ij} = \sigma(\co_i + \ro_j)$ is, the closer the corresponding element in the mixed product is to the subtropical product. In the limit, when all elements of $\mAlpha$ tend to $\infty$, we have $\mB \mixprod_{\co, \ro} \mC \to \mB \maxprod \mC$.

  We can interpret the values in parameter vectors $\co$ and $\ro$ to give the ``typical'' level of NMF or subtropical structure associated with the corresponding rows and columns. If, for example, $\co_i \ll 0$, it means that row $i$ has strong NMF-type structure, while $\co_i \gg 0$ would mean strongly subtropical structure. If $\co_i\approx 0$, then the structure is an even mixture of the two. Similarly, if $\co_i + \ro_j \gg 0$, then the element $\mA_{ij}$ has subtropical structure, and vice versa for $\co_i + \ro_j \ll 0$. This interpretation also explains why we use the tropical rank-1 model, that is, summation, instead of the standard rank-1 model $\co\ro^T$: if we calculate the product, we cannot interpret negative values of $\co_i$ or $\ro_j$ as indicative of ``typically NMF'' structure, as if both $\co_i,\ro_j < 0$, then $\co_i\ro_j > 0$, indicating subtropical structure. 
  

  Now we can define the main problem considered in this paper.
\begin{problem} \label{main:problem}
  Given an input matrix $\mA \in \rgnm$ and an integer $k>0$, find two factor matrices $\mB \in \rgnk$ and $\mC \in \rgkm$ and parameter vectors $\co \in \rgc$ and $\ro \in \rgr$ such that
  \begin{equation} \label{main:obj}
     \E(\mA, \mB, \mC, \co, \ro) = \|\mA - \defprod\|_F
   \end{equation}
   is minimized.
 \end{problem}

 Unfortunately it seems that the optimization of the above problem is hard:

 \begin{proposition}
   \label{prop:np-hardness}
   Given $\mA\in\rgnm$, $k$, $\co$, and $\ro$, finding $\mB\in\rgnk$ and $\mC\in\rgkm$ that minimize $\E(\mA, \mB, \mC, \co, \ro)$ is \NP-hard. It is also \NP-hard to find $\mB\in\rgnk$ and $\mC\in\rgkm$ that approximate $\E(\mA, \mB, \mC, \co, \ro)$ to within any polynomially computable factor.
 \end{proposition}

 The proposition is a direct consequence of the \NP-hardness of computing or approximating NMF~\cite{vavasis09complexity} or subtropical matrix factorization~\cite{karaev16capricorn}.

\section{The Algorithm}
\label{sec:algorithms}

  The algorithm \Latitude (Algorithm~\ref{alg:latitude}) finds a mixed linear--tropical matrix decomposition of the given input data.\footnote{Code is available at \url{https://people.mpi-inf.mpg.de/~pmiettin/linear-tropical/}} As input it accepts the data matrix $\mA \in \rgnm$, the rank of the sought decomposition $k\in\N$, and an integer parameter $\niter\in\N$, that determines the number of iterations of the algorithm. It returns the computed factors $\mB \in \rgnk$ and $\mC \in \rgkm$ and parameter vectors $\co \in \rgc$ and $\ro \in \rgc$. \Latitude has also one parameter, $M\in\rg$. 
  Each element in $\co$ and $\ro$ must belong to the $[-M, M]$ interval. However, in practice very high values in the parameter vectors do not make sense due to the use of the sigmoid function (see \eqref{alpha:sigmoid}) -- they would get ``smoothed out'' and make only marginal changes to the parameter matrix $\mAlpha$. For this reason for all experiments in this paper we used $M=5$, at which point $\sigma(M)=0.9933$, and there is almost nothing to be gained by increasing $M$ further.
  
   \begin{algorithm}[tbp]
   \flushleft\small%
   \caption{\Latitude}\label{alg:latitude}
   \begin{algorithmic}[1]
     \Statex \textbf{Input:} $\mA \in \rgnm$, $k\in\N$, $\niter \in\N$ 
     \Statex \textbf{Output:} $\Bbest \in \rgnk$, $\Cbest \in \rgkm$, $\cobest \in \rgc$, $\robest \in \rgr$ 
     \Statex \textbf{Parameters:} $\maxparam$ \Comment{The maximum possible value of parameter vectors. In practice 5 is a good choice}
     \Function{\Latitude}{$\mA$, $k$, $\niter$}
      \State initialize $\mB$ and $\mC$ \label{init}
      \State $\mD \gets \mB \mC - \mA$
      \State $\vf_i \gets \sum_{j=1}^m{\mD_{ij}}, \vg_j \gets \sum_{i=1}^n{\mD_{ij}}$
      \State $\vs_i \gets$ \text{index of the} $i$\text{-th smallest element of} $\vf$
      \State $\vt_j \gets$ \text{index of the} $j$\text{-th smallest element of} $\vg$
      \State $\co_i \gets \frac{i-n}{n-1}\maxparam$
      \State $\ro_j \gets \frac{j-m}{m-1}\maxparam$
      \State $\Bbest \gets \matr{B}, \Cbest \gets \matr{C}$ \Comment{Initialize best factors.}
      \State $\cobest \gets \co, \robest \gets \ro$ \Comment{Initialize best parameters.}
       \State $\bestError \gets \deferror$
       \For{$\iter \gets 1$ \textbf{to} $\niter$} \label{outer:loop}
       \For{$j \gets 1$ \textbf{to} $m$} 
       \State $[\mC^j, \ro_j] \gets \SolveMixRegression(\mA^j, \mB, \mC^j, \co, \ro_j, M)$ \label{update:C}
       \EndFor 
       \For{$i \gets 1$ \textbf{to} $n$} 
       \State $[\mB_i, \co_i] \gets \SolveMixRegression(\mA_i^T, \mC^T, \mB_i^T, \ro, \co_i, M)$ \label{update:B}
        \EndFor 
        \If{$\deferror < \bestError$} \label{check:improvement}
        \State $\Bbest \gets \matr{B}, \Cbest \gets \matr{C}$ \label{update:best:factors}
        \State $\cobest \gets \co, \robest \gets \ro$ 
        \State $\bestError \gets \deferror$ \label{update:best:error}
          \EndIf       
       \EndFor \label{outer:loop:end}
       \State \textbf{return} $\Bbest$,  $\Cbest$, $\cobest$, $\robest$
     \EndFunction
   \end{algorithmic}
 \end{algorithm}

The main idea of \Latitude is to repeatedly use a routine that solves the linear--tropical regression problem to alternatingly update  the factor matrices and the parameter vectors. Namely, when the factor matrix $\mB$ and the parameter vector $\co$ are fixed, finding the other factor matrix $\mC$ and parameter vector $\ro$ reduces to solving the problem
 \begin{equation} \label{mix:reg}
[\mC^j, \ro_j] \gets \argmin_{\vc \in \rgk,\, s \in [-M, M]} \|\mA^j - \mB \mixprod_{\co, s} \vc\|_F
 \end{equation}
 $m$ times (once per column of $\mC$). Then we fix $\mC$ and $\ro$ and do the same for $\mB$ and $\co$. This process is repeated $M$ times. The algorithm starts by initializing the factor matrices $\mB$ and $\mC$  (line~\ref{init}). This can be done by using random matrices, or, for example, by using some NMF algorithm. Starting with a ``pure'' NMF solution gives us a reasonable initial solution, and we use that initialization in our experiments. The updates to the factors and parameters are done inside the main loop (lines~\ref{outer:loop}--\ref{outer:loop:end}), where line~\ref{update:C}  updates $\mC$ and $\ro$, and line~\ref{update:B} updates $\mB$ and $\co$. On each iteration we check if the current solution $\mB$, $\mC$, $\co$, $\ro$ improves on the best one found before that (line~\ref{check:improvement}), and if it does, then we update the best solution and the best error (lines~\ref{update:best:factors}--\ref{update:best:error}).

 The function \SolveMixRegression (Algorithm~\ref{alg:solvemixregression}) solves problem \eqref{mix:reg}, and is where the actual updates to the factors and parameters are performed. It takes as input vector $\va \in \rgn$, the first factor matrix $\mB \in \rgnk$, an initial solution for the output vector $\vc \in \rgk$, the column parameter vector $\co$, the starting value for the row parameter element $t$, and the number $M>0$ that defines the range of the values in the parameter vectors. It returns the updated versions of the vector $\vc$ and the element $t$. Finding the global minimum of \eqref{mix:reg} with respect to both $\vc$ and $t$ is hard, and hence we update them separately. In fact, even when the parameter $t$ is fixed, optimizing \eqref{mix:reg} with respect to $\vc$ is problematic. To see that, let us first rewrite \eqref{mix:reg} for a fixed value of $t$. It becomes
 \begin{equation} \label{reg:fixed}
   \argmin_{\vc \in \rgk} \|\va - (\sigma(\co+t) \mB \maxprod \mC + (1-\sigma(\co+t)) \mB \vc)\|_F \;.
 \end{equation}
 
 For every $1 \le i \le n$ denote by $\varphi(i, \vc)$ the index of the largest element in the vector $\mB_i \Hadamard \vc^T$, where $\Hadamard$ is the element-wise (Hadamard) product. We have
 \begin{equation}
   \label{eq:reg_transformed}
   \begin{split}
    &\quad\; \mAlpha_i \mB_i \maxprod \vc + (1 - \mAlpha_i) \mB_i \vc \\
    &= \mAlpha_i \max_s\lbrace \mB_{is}\vc_s\rbrace + (1 - \mAlpha_i) \sum_{s}\mB_{is}\vc_s   \\
    &= \mB_{i\varphi(i, \vc)}\vc_{\varphi(i, \vc)} + (1 - \mAlpha_i) \sum_{s \ne \varphi(i, \vc)}\mB_{is}\vc_s \;,
   \end{split}
 \end{equation}
and hence the problem \eqref{reg:fixed} is transformed into
 \begin{equation} \label{almost:reg}
\argmin_{\vc \in \rgk} \|\va - \mY(\vc) \vc\|_F , \;
  \mY(\vc)_{ij} = \begin{cases}
           1 & j = \varphi(i, \vc)   \\
           1-\mAlpha_i &  \text{otherwise \;.}
         \end{cases}
 \end{equation}

    \begin{algorithm}[tbp]
   \flushleft\small%
   \caption{\SolveMixRegression}\label{alg:solvemixregression}
   \begin{algorithmic}[1]
     \Statex \textbf{Input:} $\va \in \rgn$, $\mB \in \rgnk$, $\vc \in \rgk$, $\co \in \rgc$, $t \in \R$, $M>0$ 
     \Statex  \textbf{Output:} $\vc \in \rgk$, $t \in \R$  
     \Function{\SolveMixRegression}{$\va$, $\mB$, $\vc$, $\co$, $t$, $M$}
     \State $\mX_i \gets \mB_i \Hadamard \vc^T$ \label{getX}
     \State $\mAlpha \gets \sigma(\co + t)$ 
     \State $\mT_{ij} \gets          \begin{cases}
           1 & j = \argmax_{1 \le s \le k} \mX_{is}   \\
           1-\mAlpha_i &  \text{otherwise}
         \end{cases}$
         \State $\mY \gets \mB \Hadamard \mT$ \label{getY}
         \State $\vc \gets \argmin_{\vp \in \rgk} \|\va - \mB \vp\|_F$ \label{getC}
         \State $t \gets \argmin_{s \in [-M, M]} \|\va - \mB \mixprod_{\co, s} \vc\|_F$ \label{get:param}
       \State \textbf{return} $\vc$,  $t$
     \EndFunction
   \end{algorithmic}
 \end{algorithm}
 
If the coefficient matrix $\mY(\vc)$ did not depend on $\vc$, \eqref{almost:reg} would become a standard nonnegative linear regression problem. Unfortunately, the dependence of $\varphi(i, \vc)$ on $\vc$ is very complex, and hence it is hard to solve \eqref{almost:reg} directly. In order to overcome this obstacle, we use another heuristic, that is we fix the coefficient matrix $\mY(\vc)$, and assume it to be independent from $\vc$. Under these assumptions $\vc$ can be found using a standard nonnegative linear regression algorithm. We use the MATLAB built-in \lsqnonneg. The matrix $\mY$ is built on lines~\ref{getX}--\ref{getY}, and the vector $\vc$ is found on line~\ref{getC}. Finally, on line~\ref{get:param} we update the parameter $t$. This is done using the binary search on the interval $[-M, M]$ for the point where the derivative with respect to $t$ is close to 0.

\paragraph{Time complexity. }

    Running \Latitude comprises of executing \NMF to initialize the factors, and then repeatedly updating them, as well as the parameters, using the \SolveMixRegression routine. For each $i=1,\ldots,n$ and $j=1,\ldots,m$, \SolveMixRegression is called $\niter$ times. In order to estimate the complexity of \SolveMixRegression it suffices to consider the case when it is called to update $\mC$ and $\ro$ as the alternate case is analogous (one just needs to replace $n$ by $m$). Computing the matrix $\mY$  (lines~\ref{getX}--\ref{getY}) and finding $t$ (line~\ref{get:param}) take time $O(nk)$ each; the latter one because it is enough to make a finite number of steps of the binary search. Thus, if we denote by $\Gamma(n, k)$ the complexity of solving the nonnegative linear regression problem, then the running time of \SolveMixRegression would be given by  $O(nk) + \Gamma(n, k)$. Since we use \NMF to initialize the factors, the runtime of \Latitude depends on what NMF algorithm is called. If we denote the complexity of \NMF by $\Pi(n, m, k)$, then the total complexity of \Latitude is $\niter m(O(nk) + \Gamma(n, k)) + \niter n(O(mk) + \Gamma(m, k)) + \Pi(n, m, k) = O(\niter nmk) + \niter m\Gamma(n, k) + \niter n\Gamma(m, k) + \Pi(n, m ,k)$.

    Using \texttt{lsqnonneg} for the nonnegative regression and denoting its average number of iterations by $\ell$ as above, we have that $\Gamma(n, k) = O(\ell nk^2)$. Using projected ALS algorithm~\cite{cichocki09nonnegative} for the NMF, each iteration takes $O(nk^2+mk^2+nmk)$ time, and we denote the expected number of iterations of the \NMF algorithm by $t$. With these choices, the total time complexity becomes $O\bigl(\niter k(nm + lnmk) + tk(nk+mk+nm)\bigr)$. Importantly, this is linear in the dimensions of the input matrix.

    For actual runtime on various real-world dataset see Appendix~\ref{app:runtime}.


\section{Experimental Evaluation}
\label{sec:experiments}
In this section we test \Latitude on both synthetic and real-world data, in order to verify how well it can recover the mixed tropical-linear structure. We also compare it against various benchmark matrix factorization methods.
\subsection{Other methods.}
\label{sec:exp:other-methods}
Since \Latitude is designed to work with data that has a mixture of NMF and SMF structures, it is important to compare against algorithms that target them both. There is a multitude of NMF algorithms, but in this paper we use MATLAB's default implementation \nnmf, to which we will refer simply as \NMF. We will also compare against \SVD in order to get a comparison to optimal rank-$k$ decomposition.
Unlike \NMF or \SVD, the subtropical matrix factorization is a quite new direction of research, and to the best of our knowledge there are only two available algorithms: \Cancer \cite{karaev16cancer} and \Capricorn \cite{karaev16capricorn}. Of these, \Cancer is more suitable due to its ability to handle Gaussian noise, and hence we chose it over \Capricorn.
\subsection{Synthetic Experiments.}
\label{sec:synth-exper}
The purpose of the synthetic experiments is to verify that the proposed algorithms are actually capable of recovering the sought structure when the data conforms to the mixed tropical-linear model. First we  generate the data using the mixed tropical-linear structure, then add some Gaussian noise, and finally run the methods to see how much structure they can recover. Unless stated otherwise, the matrices are of size $1000\times800$ with true rank 10 and values drawn uniformly at random from the $[0, 1]$ interval. The factor density is by default 20\%, and the standard deviation of the Gaussian noise is 0.01. In order to make sure that after applying the noise the data remains nonnegative, we truncate all values below 0. The parameter vectors $\co$ and $\ro$ are drawn uniformly at random from the $[-5, 5]$ interval. For the pure subtropical and NMF structure experiments we did not use parameters, but rather multiplied the factors directly. The reconstruction error is always measured against the original, noise-free matrix.
\begin{figure*}[t]
  \centering
  \subfloat[Varying noise with pure subtropical data.]{%
    \includegraphics[width=\subfigwidth]{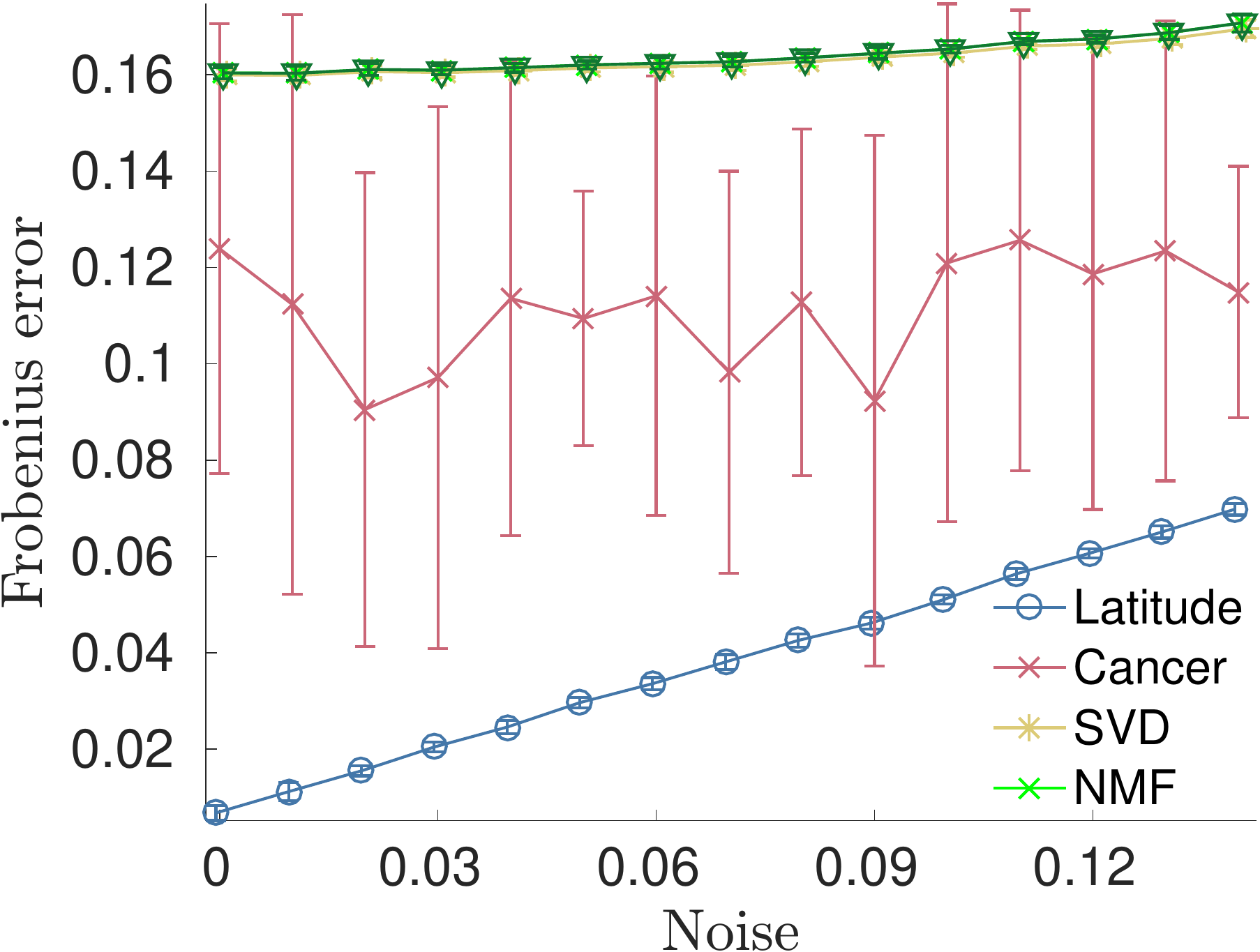}
    \label{trop:noise}
  }
  \hspace{\subfigspace}
  \subfloat[Varying noise with pure NMF data.]{%
    \includegraphics[width=\subfigwidth]{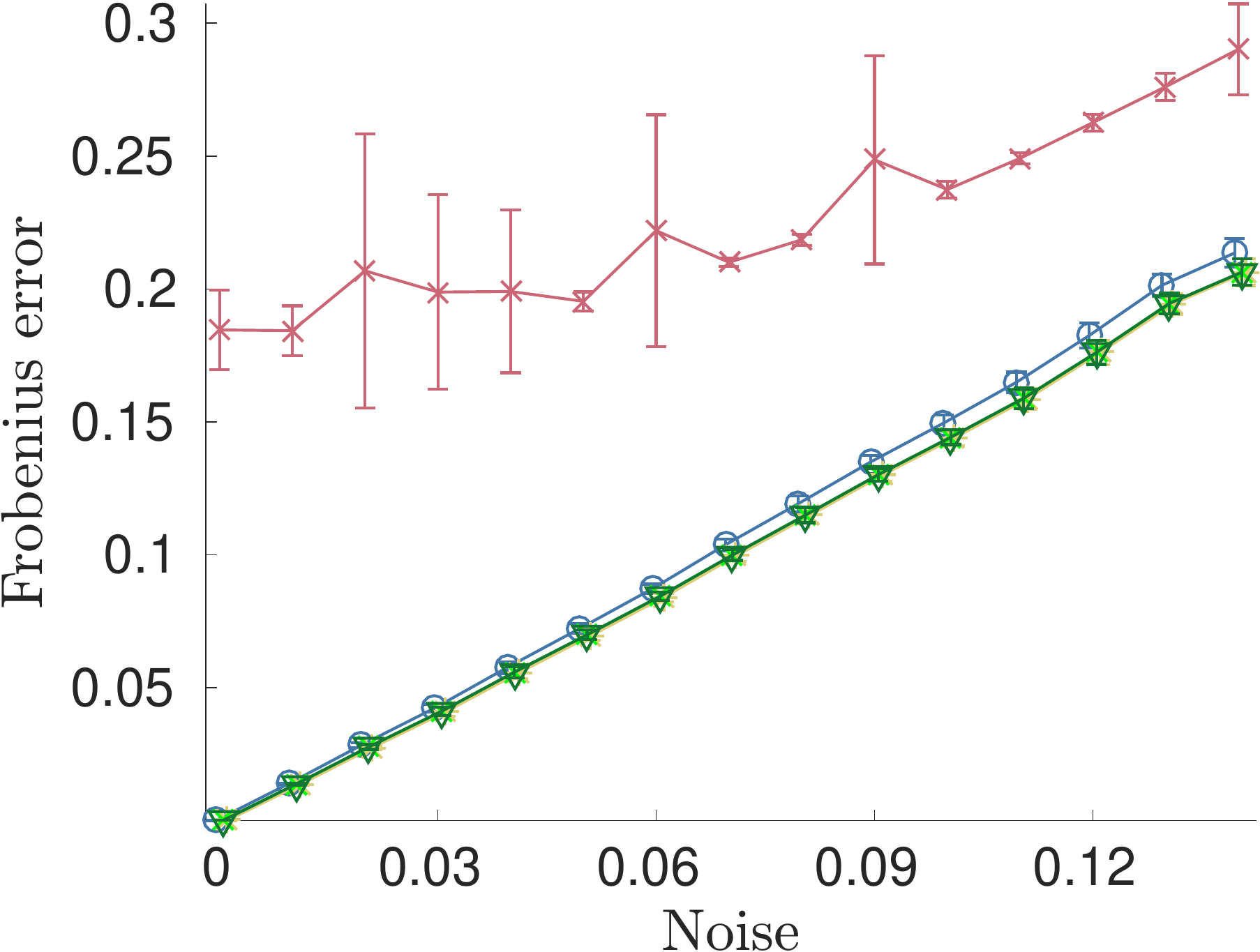}
    \label{nmf:noise}
  }
  \hspace{\subfigspace}
  \subfloat[Varying noise with mixed data.]{%
    \includegraphics[width=\subfigwidth]{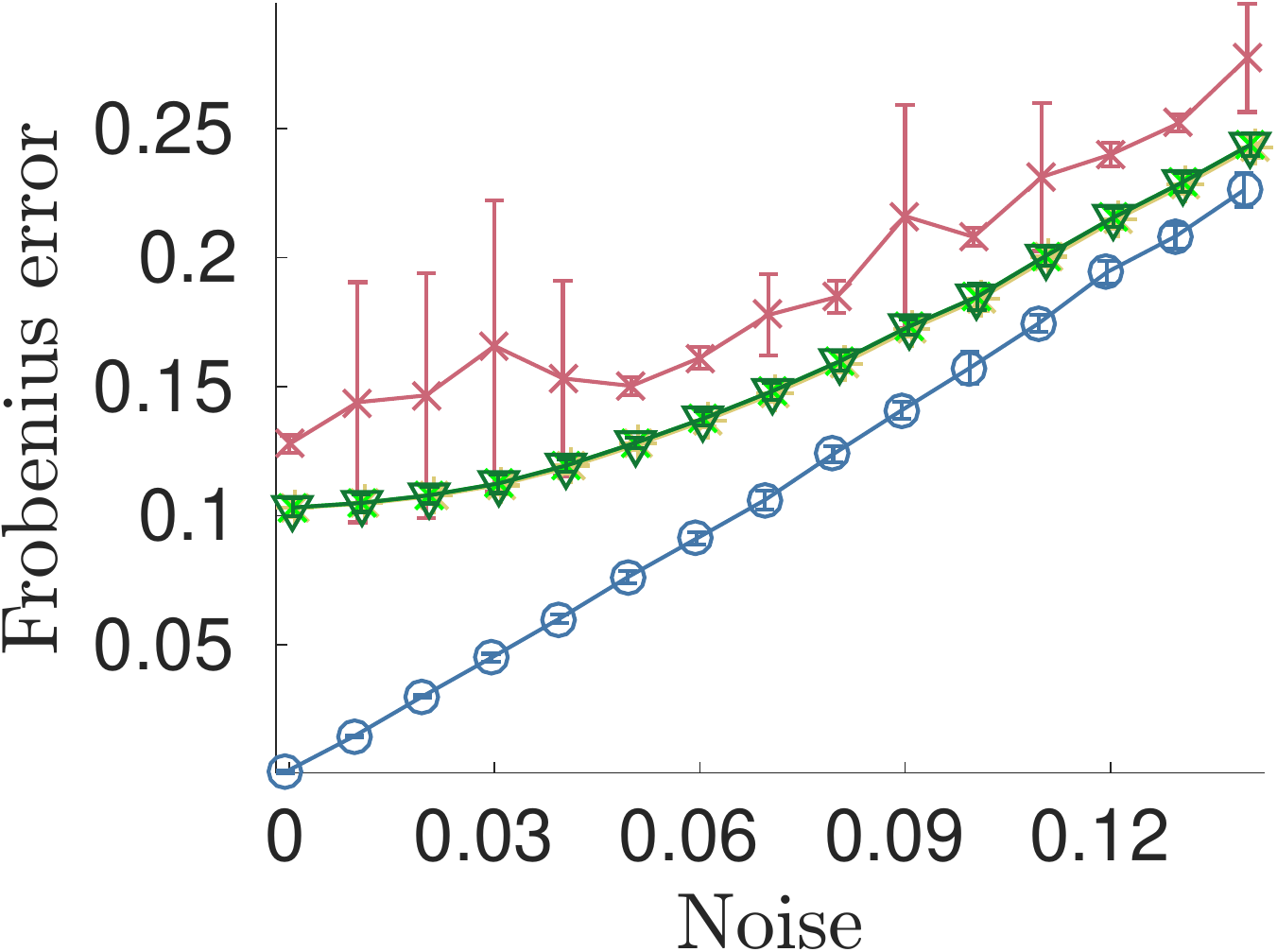}
    \label{mixed:noise}
  }
  \hspace{\subfigspace}
  \subfloat[Varying factor density with mixed data.]{%
    \includegraphics[width=\subfigwidth]{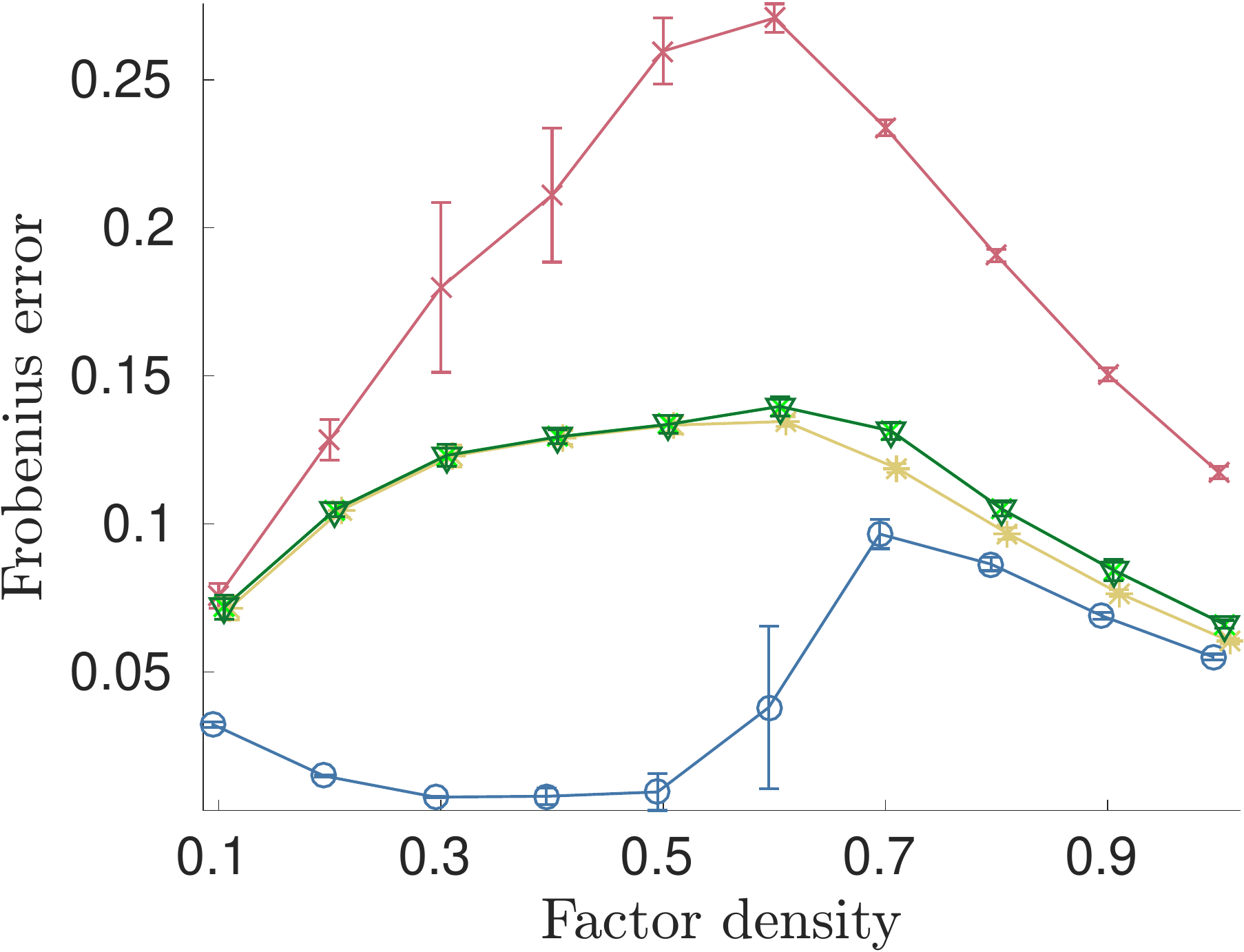}
    \label{density}
  }
  \subfloat[Varying rank with mixed data.]{%
    \includegraphics[width=\subfigwidth]{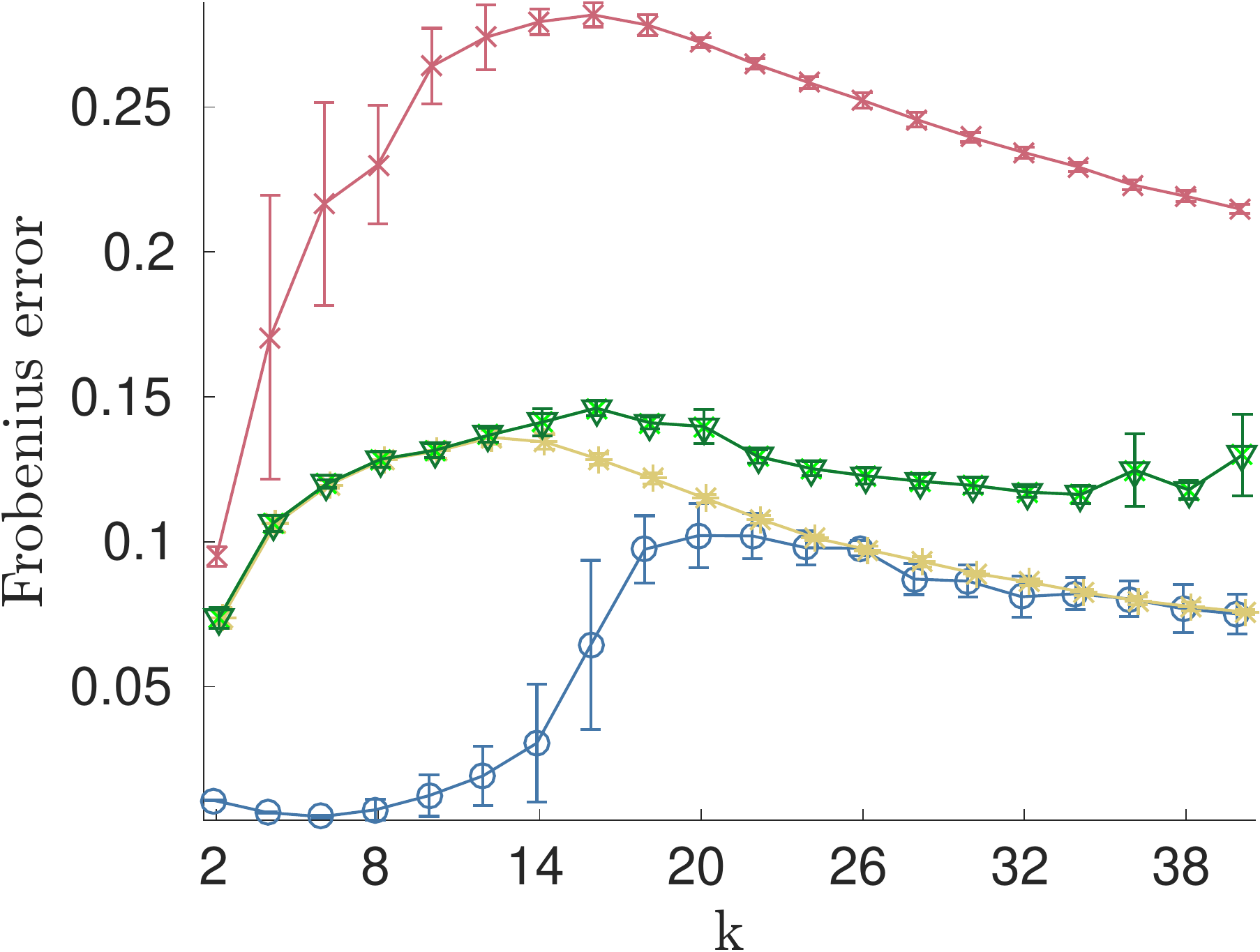}
    \label{rank}
  }
  \hspace{\subfigspace}
  \subfloat[Varying rank with mixed data a high level of noise.]{%
    \includegraphics[width=\subfigwidth]{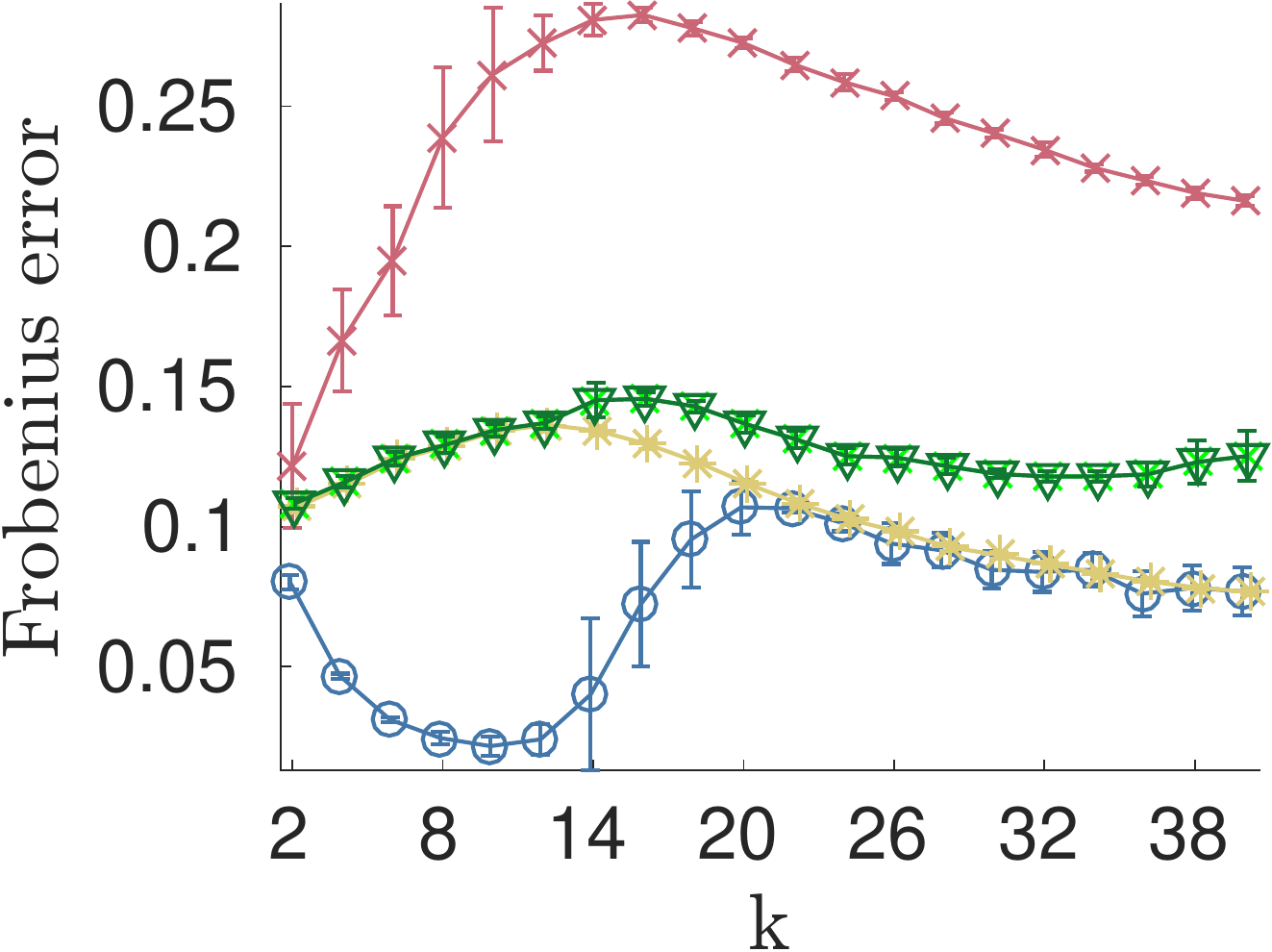}
    \label{dim:high:noise}
  }
  \caption{\textbf{Reconstruction errors on synthetic data.} The $x$-axis represents the varying parameter and the $y$-axis the Frobenius error. All results are averages over 10 random matrices and the width of the error bars is twice the standard deviation.}\label{fig:synth:reconstruct:frob}
\end{figure*} 
\paragraph{Varying noise with pure subtropical data.}  (Fig~\ref{trop:noise}) This experiment tests how well various methods can recover the pure subtropical structure, that is, the extreme case of all parameters being set to $\infty$. The data is generated by multiplying the factors using the subtropical matrix product. We varied the standard deviation of the Gaussian noise from 0 to 0.14 with increments of 0.01. \Latitude is clearly the best method, followed by \Cancer, and \NMF and \SVD come close together in the last place. The reason why \Latitude beats \Cancer on its own kind of data is that it has more leeway in choosing what structure to use, thus being able to fit everywhere where \Cancer approximates the data well, but also deviate from the pure subtropical model when needed. \NMF and \SVD do not seem to find much structure in this experiment. In this and some other experiments \SVD and \NMF produce similar reconstruction errors, which sometimes makes their lines hard to distinguish.
\paragraph{Varying noise with pure NMF data.} (Fig.~\ref{nmf:noise}) This setup is analogous to the previous one, except now the data was generated using the pure NMF structure. Here, \NMF and \SVD are performing very well, as is expected as the data is generated with the NMF structure. \Latitude, although having been initialized by \NMF, only achieves the same results for zero level of noise -- then its results start to slowly deteriorate. The cause of this is that it overfits to the noise. Nevertheless, \Latitude's results are not much worse than \NMF or \SVD, and hence it is definitely applicable to datasets that exhibit the pure NMF structure. Meanwhile \Cancer is the worst of the methods, which is expected given that the data has pure NMF rather than subtropical structure.
\paragraph{Varying noise with mixed data.}  (Fig.~\ref{mixed:noise}) Here we test the actual mixed model by using parameters drawn uniformly at random from the $[-5, 5]$ interval. This means that the expected value of $\co_i + \ro_j$ is 0, which corresponds to the midpoint between the NMF and subtropical structures. The randomness ensures that both structures are present in the data. Here \NMF and \SVD perform much better than for the pure subtropical case, but \Latitude is nevertheless the best method by a big margin, which demonstrates the advantages of combining both models. 
\paragraph{Varying factor density with mixed data.}  (Fig.~\ref{density}) Here we varied the factor density from 10\,\% to 100\,\% with increments of 10\,\%. Again, we have \Latitude as the best method. There is a peculiar bump on its curve at the very low density level. It can be explained by noise having more influence on sparse data, since then the data/noise ratio is worse.
\paragraph{Varying rank with mixed data.}  (Fig.~\ref{rank}) Here we varied the actual rank of the data from 2 to 40 with increments of 2. The factor density was kept at 50\,\%. As in previous experiments, \Latitude performs significantly better, especially for lower ranks. Appendix~\ref{app:vary_k} contains another variation of this setup.
\paragraph{Varying rank with mixed data with a high level of noise.}  (Fig~\ref{dim:high:noise}) Same setup as above, but with a higher level of noise (standard deviation 0.07). \Latitude again performs much better than other methods, albeit having a weird bump for lower ranks. Here again it is explained by lower rank data having also lower density, which exacerbates the effect of the noise.
It is worth mentioning that, with the exception of the subtropical data test (Fig~\ref{trop:noise}), \Cancer gives the highest reconstruction error. This is not surprising since it aims at recovering the subtropical structure, which is no more present in the data in its pure form.
\subsection{Real-World Experiments.}
\label{sec:real-world-exper}
Now that we have evidence that \Latitude can extract the mixed tropical-linear structure when it is present in the data, we want to see if this kind of structure is also present ``in the wild''. For that we ran all the competing algorithms on various real-world datasets. First we briefly describe the data, then provide the numerical comparison of the results of the algorithms, followed by some example results.
\paragraph{Datasets.}
\label{sec:real:data}
Rather than using raw data, we did some common preprocessing for the real-world datasets. To ensure nonnegativity, we subtract from each column its smallest element. In addition, to make the data more uniform, we divide each column by its standard deviation. These steps are performed on all matrices except \News, for which we use the TF-IDF model. 
\Worldclim was obtained from the global climate data repository.\!\footnote{The raw data is available at \url{http://www.worldclim.org/}, accessed 18 July 2017.} It describes historical climate data across different geographical locations in Europe. Columns represent minimum, maximum, and average temperatures and precipitation in different months, and rows ($2575$) are \by{50}{50} kilometer squares of land where measurements were made. Although temperatures and precipitation are seemingly heterogeneous and have different numeric scales, they are equally important in determining the climate type. To be able to use both of them together, prior to performing the standard preprocessing as with other matrices, we subtract from every column its mean.
\NPAS is a nerdiness personality test that uses different attributes to determine the level of nerdiness of a person.\!\footnote{The dataset can be obtained on the online personality website \url{http://personality-testing.info/_rawdata/NPAS-data.zip}, accessed 18 July 2017.} It contains answers by 1418 respondents to a set of 36 questions that asked them to self-assess various statements about themselves on a scale of 1 to 7. We preprocessed \NPAS  analogously to \Worldclim.
\Eigenfaces is a subset of the Extended Yale Face collection of face images~\cite{georghiades2000few}. It consists of $222$ \by{32}{32} pixel images under different lighting conditions. We used a preprocessed data by Xiaofei He et al.\!\footnote{\url{http://www.cad.zju.edu.cn/home/dengcai/Data/FaceData.html}, accessed 18 July 2017} We selected a subset of pictures with lighting from the left.
\News is a subset of the {20 Newsgroups} dataset,\!\footnote{\url{http://qwone.com/~jason/20Newsgroups/}, accessed 18 July 2017} containing the usage of 800 words over 400 posts for 4 newsgroups.\!\footnote{The authors are  grateful to Ata Kab{\'a}n for pre-processing the data, see~\cite{miettinen09matrix}.} Before running the algorithms we transformed the data to TF-IDF values, and scaled by dividing each entry by the greatest entry in the matrix.
\HPI is a land registry house price index.\!\footnote{Available at \url{https://data.gov.uk/dataset/land-registry-house-price-index-background-tables/}, accessed 18 July 2017} Rows ($253$) represent months, columns ($177$) are locations, and entries are residential property price indices. 
Further information about these datasets is available in Table~\ref{tab:real:specs_all} in the Appendix.
 \paragraph{Numerical experiments.}
 The reconstruction errors for all the real-world experiments are shown in Table~\ref{tab:real:world:error}. \Latitude and \SVD are competing for the first place, with \Latitude having the best reconstruction error in 2 datasets and \SVD in 3. All other methods fall significantly behind. It is worth mentioning that \SVD has an advantage in that it its factors are not restricted to nonnegative values. One can also argue that \Latitude has more degrees of freedom due to having one additional dimension of parameters. For this reason we also test the truncated version, called \Lattrunc, that was run with $k-1$ dimensions. It is still the third best method (after \SVD and \Latitude), beating both \NMF and \Cancer by a wide margin. Given these results we can conclude that the mixed tropical-linear structure is present in the datasets that we tested, and that \Latitude is an appropriate algorithm to extract this structure.
 \setlength{\tabcolsep}{0.5em}
 \begin{table}[tb] 
   \centering
   \caption{Reconstruction error for real-world datasets.}
\label{tab:real:world:error}
   \begin{tabular}{@{}lRRRRR@{}}
     \toprule
                      & \text{\Worldclim} & \text{\NPAS} & \text{\Eigenfaces} & \text{\News} & \text{\HPI}  \\
     $k=$         & 10      & 10      & 40      & 20      & 15 \\
     \midrule
     \Latitude  & \bf{0.023} & \bf{0.207} & 0.157 & 0.536 & 0.016     \\
     \Lattrunc  & 0.025 & 0.213 & 0.158 & 0.541 & 0.017     \\
     \SVD        & 0.025 & 0.209 & \bf{0.140} & \bf{0.533} & \bf{0.015}     \\
     \NMF       & 0.080 & 0.223 & 0.302 & 0.541 & 0.124     \\
     \Cancer     & 0.066 & 0.237 & 0.205 & 0.554 & 0.026     \\
     \bottomrule
   \end{tabular}
 \end{table}
\paragraph{Interpretation.}
In order to validate that our approach also provides interpretable results, we study the results with \Worldclim and \Eigenfaces in more detail. We used the ranks from Table~\ref{tab:real:world:error}. NMF is used in climate models~\cite{paatero94positive}, so we would expect this data to have mostly NMF structure, but certain phenomena, such as rainfall, and certain areas, such as mountains or coastal sites, can well have more subtropical structure.
To validate this intuition, we can study the parameter vectors $\co$ and $\ro$ and matrix $\mAlpha = \bigl(\sigma(\co_i + \ro_j)\bigr)_{ij}$. For the \Worldclim data, these are depicted in Figure~\ref{fig:worldclim_param}. Recall that for the parameters, negative values indicate NMF-type structure, while positive values indicate subtropical-type structure. Vector $\co$ corresponds to the geographical locations, and its values are plotted in a map in Figure~\ref{fig:wordclim_param:co}. As we expected, most of the data has NMF-type structure (depicted as blue), but especially Lapland, Portugal, and some mediterranean coastlines have more subtropical-type structure. These areas probably have some dominating climate phenomena, for example, heavy rainfall or low temperatures, that is best explained using subtropical structure.
\begin{figure*}[t]
  \centering
  \subfloat[Vector $\co$]{%
    \label{fig:wordclim_param:co}%
    \includegraphics[width=0.75\subfigwidth]{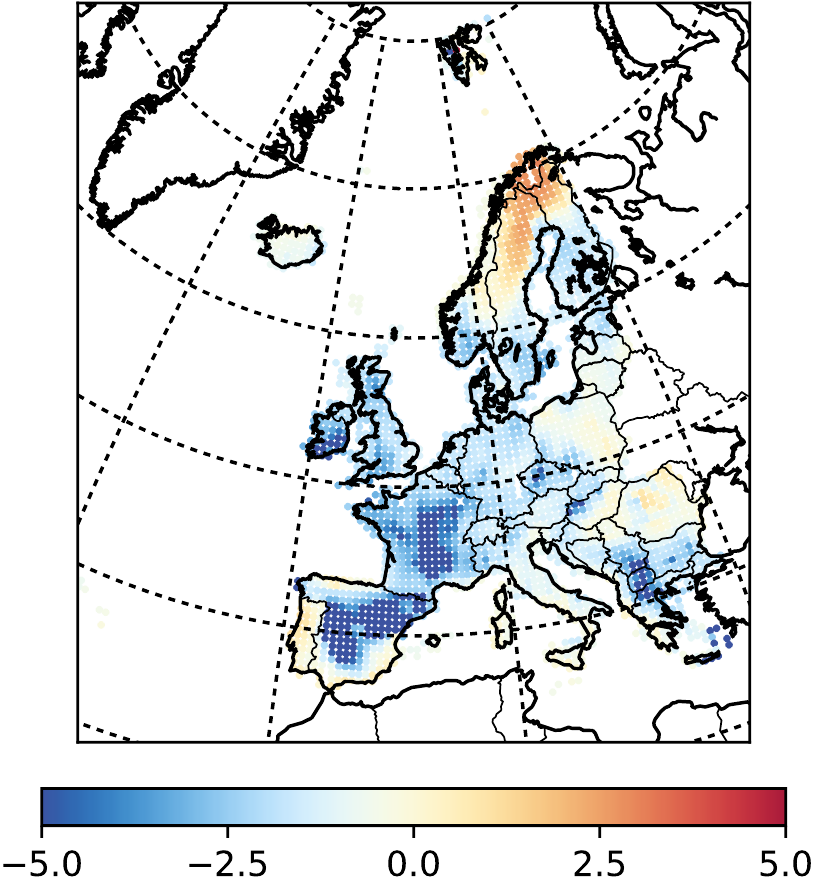}%
  }\qquad
  \subfloat[Vector $\ro$]{%
    \label{fig:worldclim_param:ro}%
    \includegraphics[width=\subfigwidth]{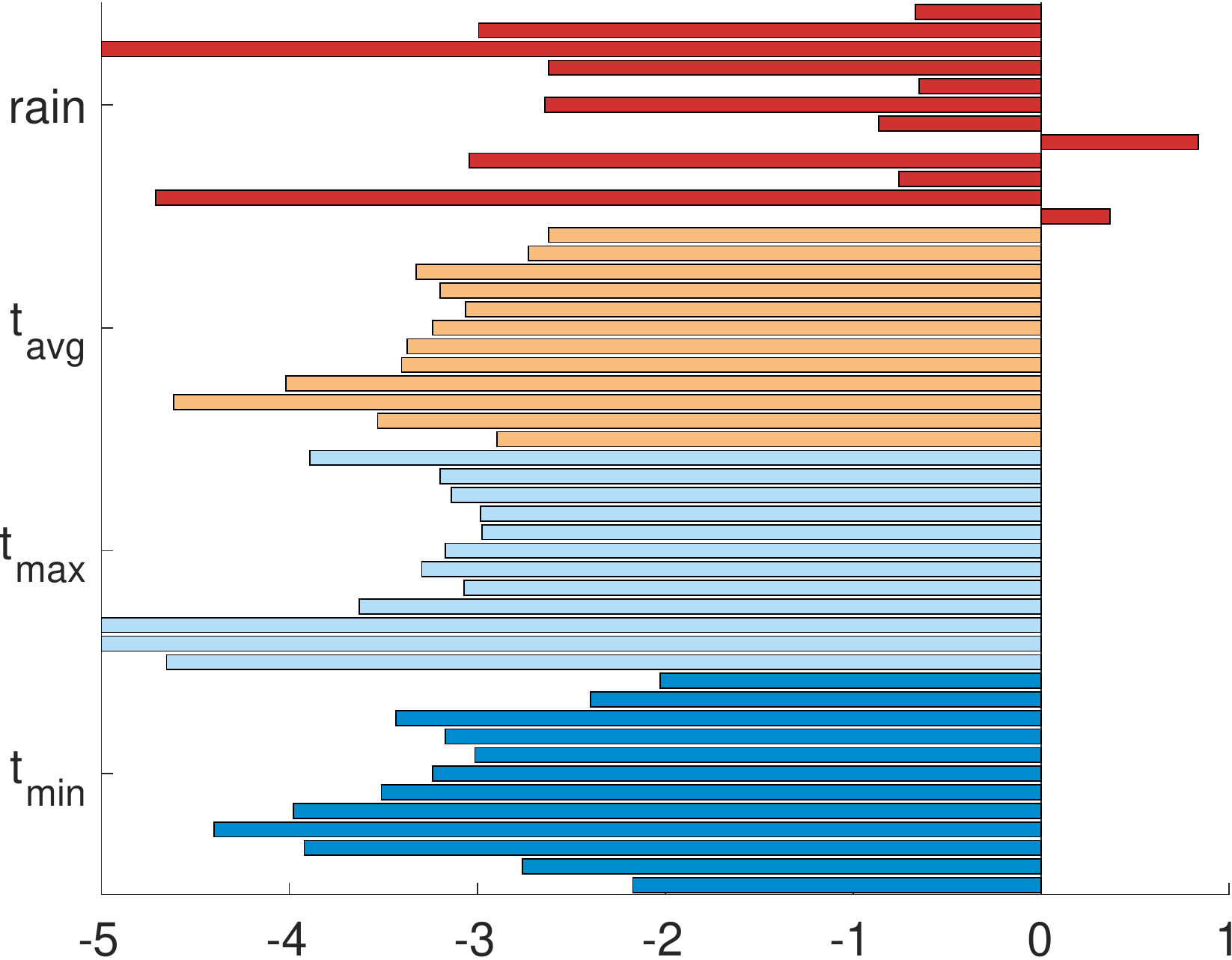}%
  } \qquad
  \subfloat[Matrix $\mAlpha$]{%
    \label{fig:worldclim_param:alpha}
    \includegraphics[width=\subfigwidth]{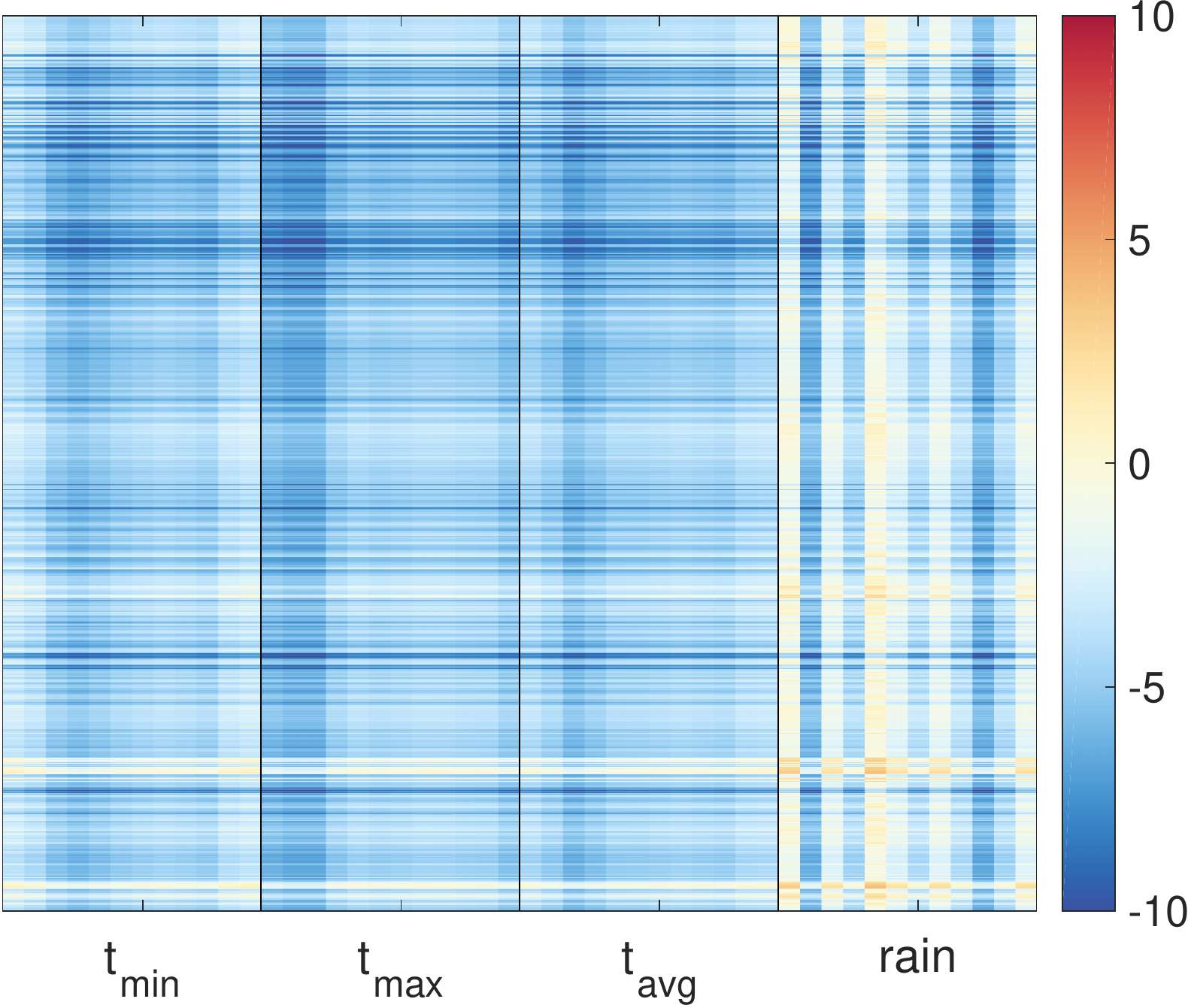}%
  }
  \caption{Visualizations for the parameters in the decomposition of \Worldclim. (a) Values in vector $\co$ plotted in a map. (b) Values in vector $\ro$ shown as a bar plot. The variables are divided in four groups of twelve months corresponding to minimum, maximum, and average temperature, and precipitation (\tmin, \tmax, \tavg, and \preci, respectively). January is always at the bottom. (c) The matrix $\mAlpha = \bigl(\sigma(\co_i + \ro_j)\bigr)_{i,j}$. Columns are divided in four groups of twelve months, as in (b). January is always at the left.}
  \label{fig:worldclim_param}
\end{figure*}
Vector $\ro$ corresponds to the climate variables. The values in $\ro$ are shown in Figure~\ref{fig:worldclim_param:ro}, where we can see that most variables are negative, that is, they have NMF-type structure. Precipitation is an exception, as the precipitation variables for January and May are in fact positive, indicating more subtropical-type structure. The complete parameter matrix $\mAlpha$ is shown in Figure~\ref{fig:worldclim_param:alpha}. Most elements in the factorization have medium to strong NMF-type structure, but there exist also elements with more subtropical-type structure.
The vector $\co$ for the \Eigenfaces data corresponds to the pixels and is depicted in Figure~\ref{fig:faces:co}. It is clear that the dominating features of faces -- eyes, nose, and mouth, are best expressed using subtropical-type structure, while the other parts are better explained using NMF-type structure. This is to be expected, as the subtropical areas are those where lighting has the largest effects (either as bright areas, or areas in shadows, depending on the direction of the light). These extremes are often easiest to describe using the subtropical structure.

\begin{figure*}[t]
  \centering
  \subfloat[Vector $\co$]{%
    \label{fig:faces:co}%
    \includegraphics[width=0.9\subfigwidth]{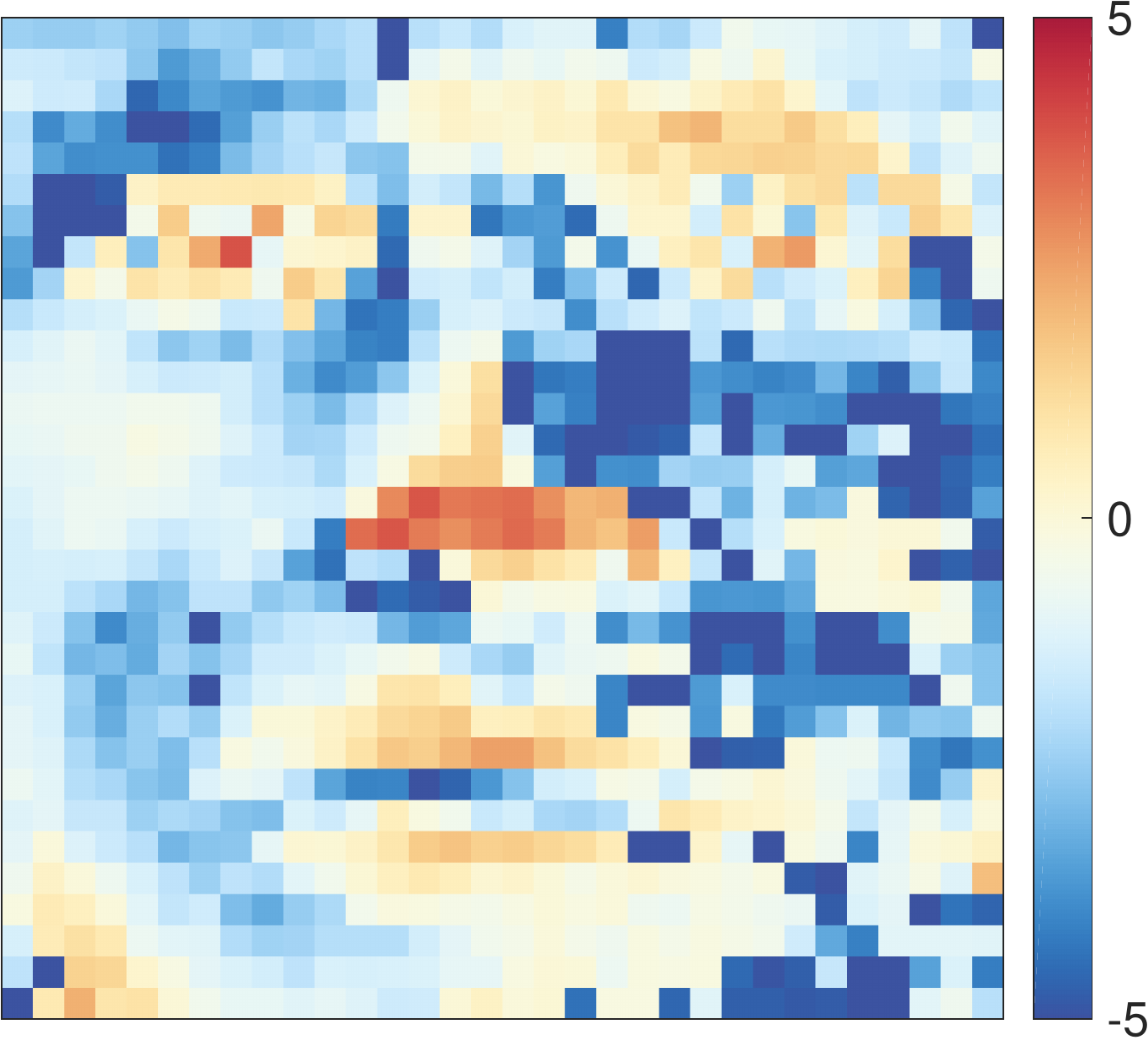}
  }\qquad
  \subfloat[Matrix $\mAlpha$]{%
    \label{fig:faces:alpha}%
    \includegraphics[width=\subfigwidth]{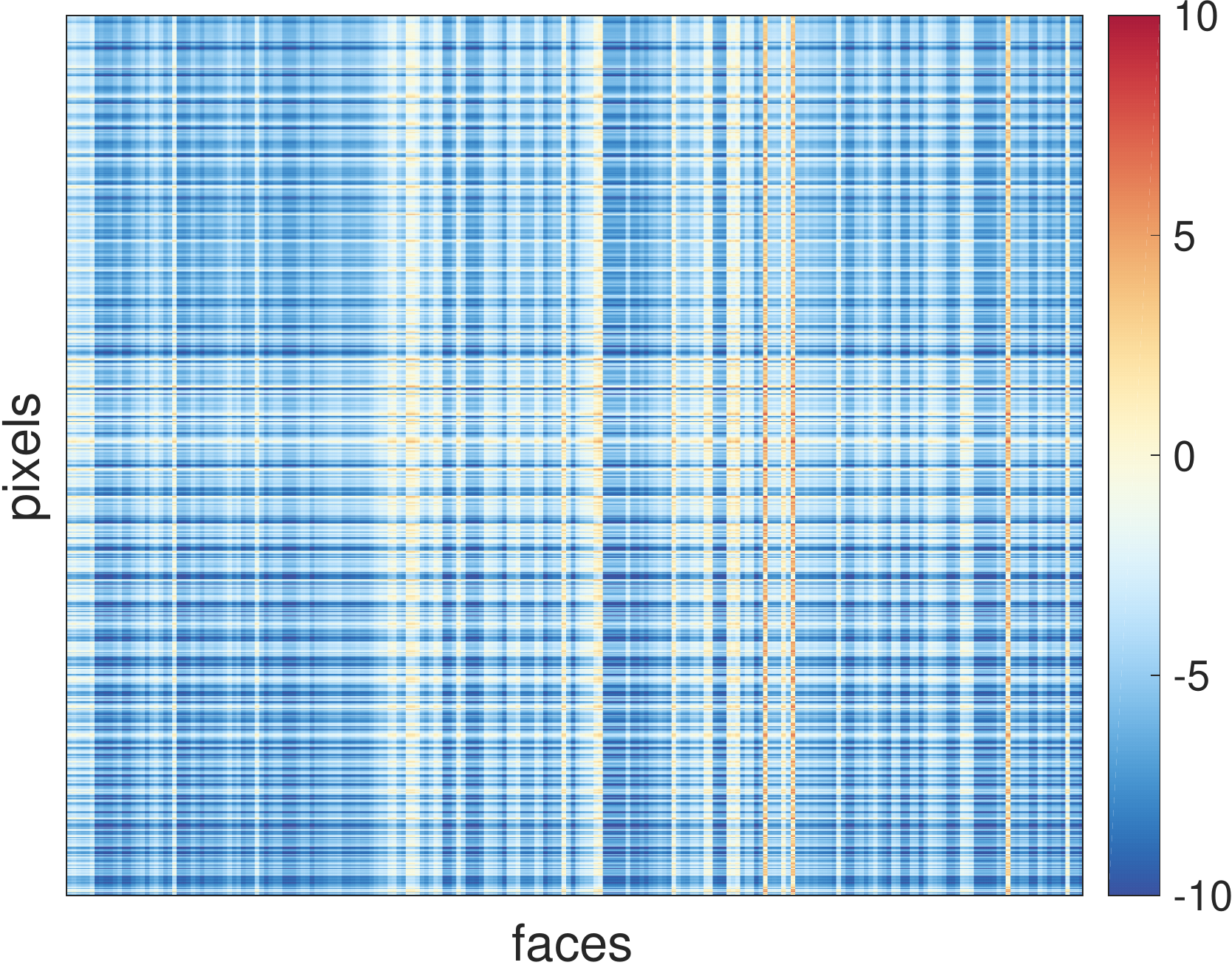}
  }\qquad
   \subfloat[Columns of $\mB$]{%
    \label{fig:faces:B}%
    \begin{tabular}[b]{@{}p{.4\subfigwidth}@{\hspace{1pt}}p{.4\subfigwidth}@{}}
      \includegraphics[width=\linewidth]{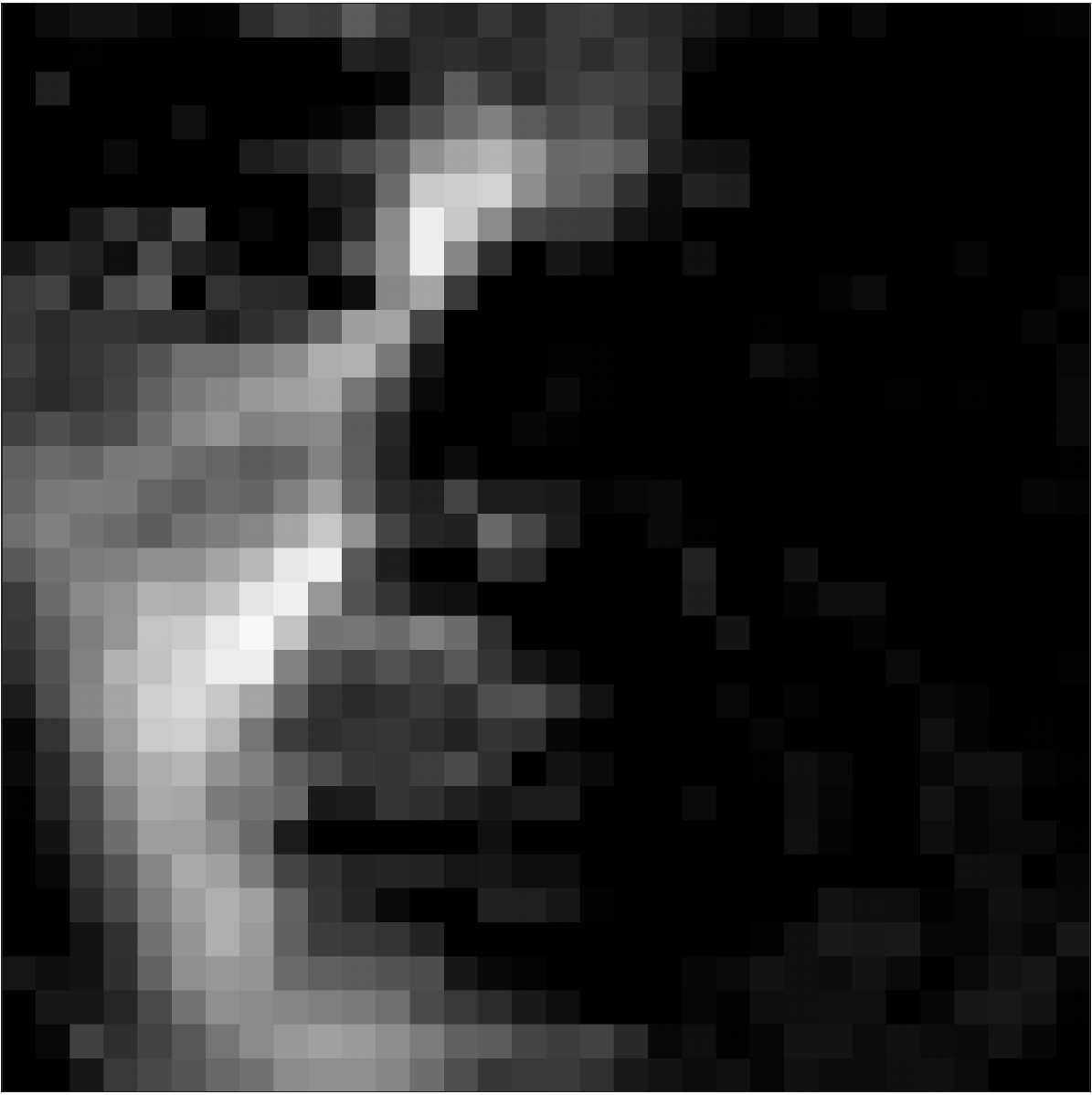} &
                                                                   \includegraphics[width=\linewidth]{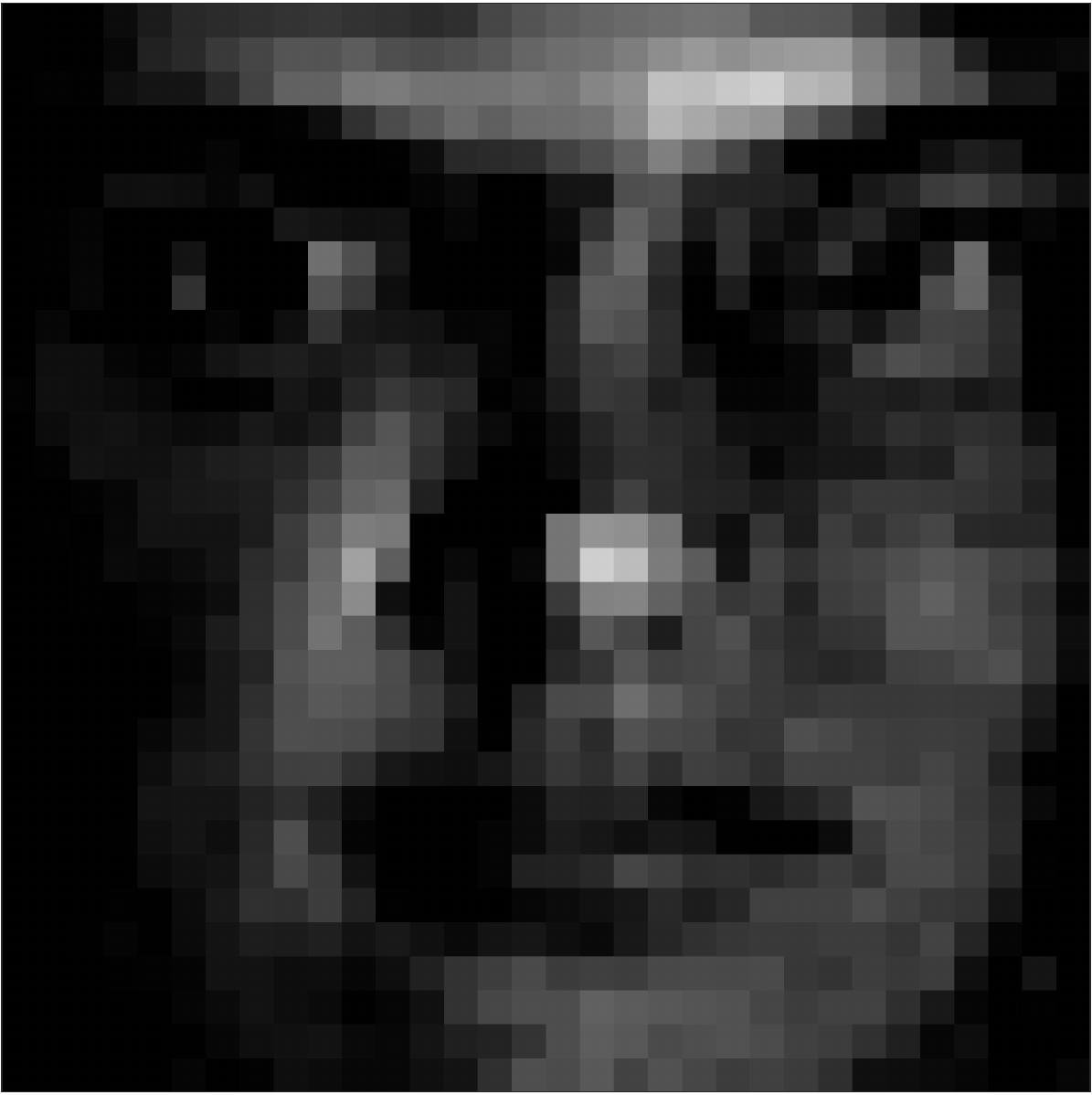} \\[-2pt]
      \includegraphics[width=\linewidth]{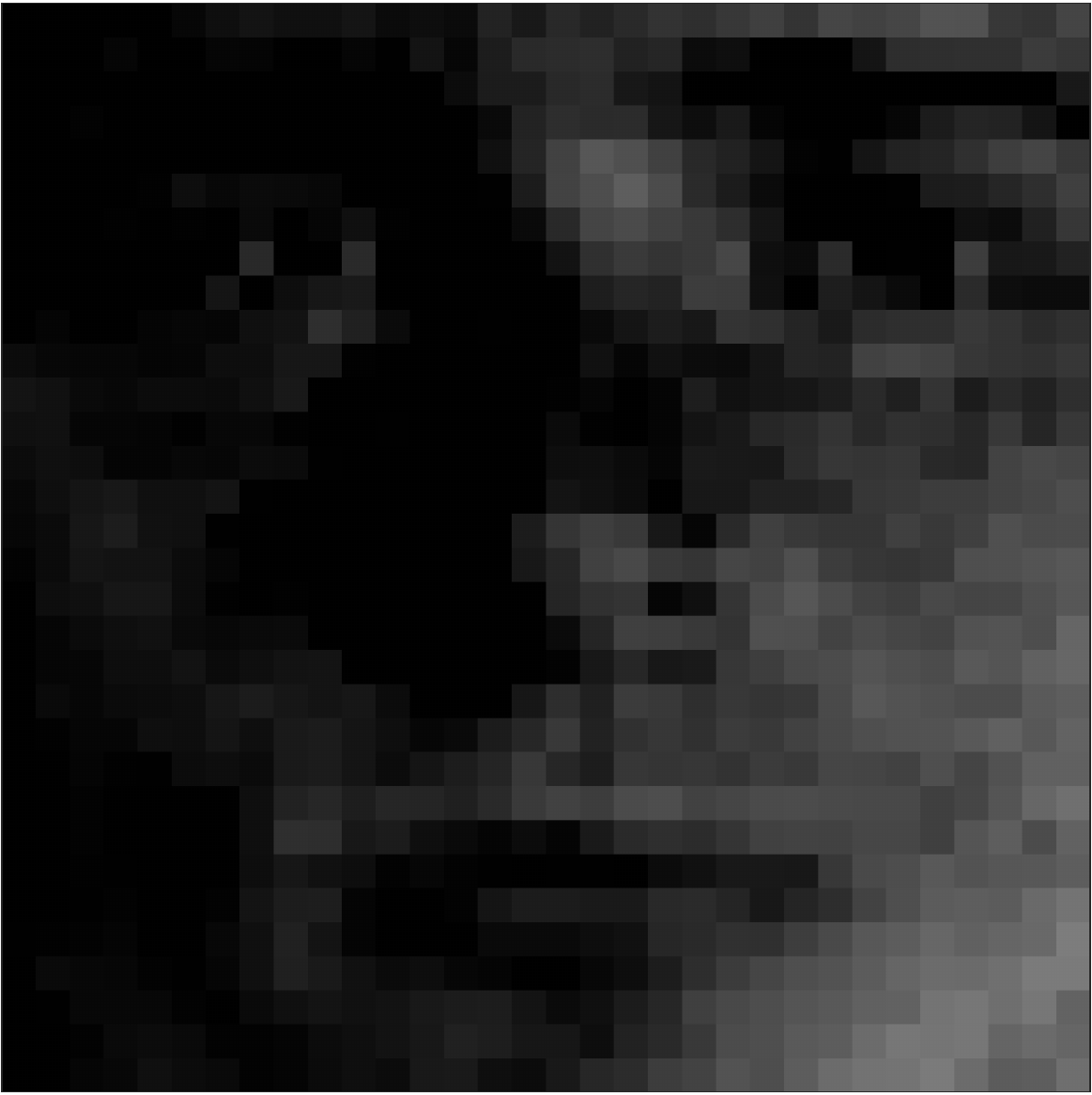} &
                                                                   \includegraphics[width=\linewidth]{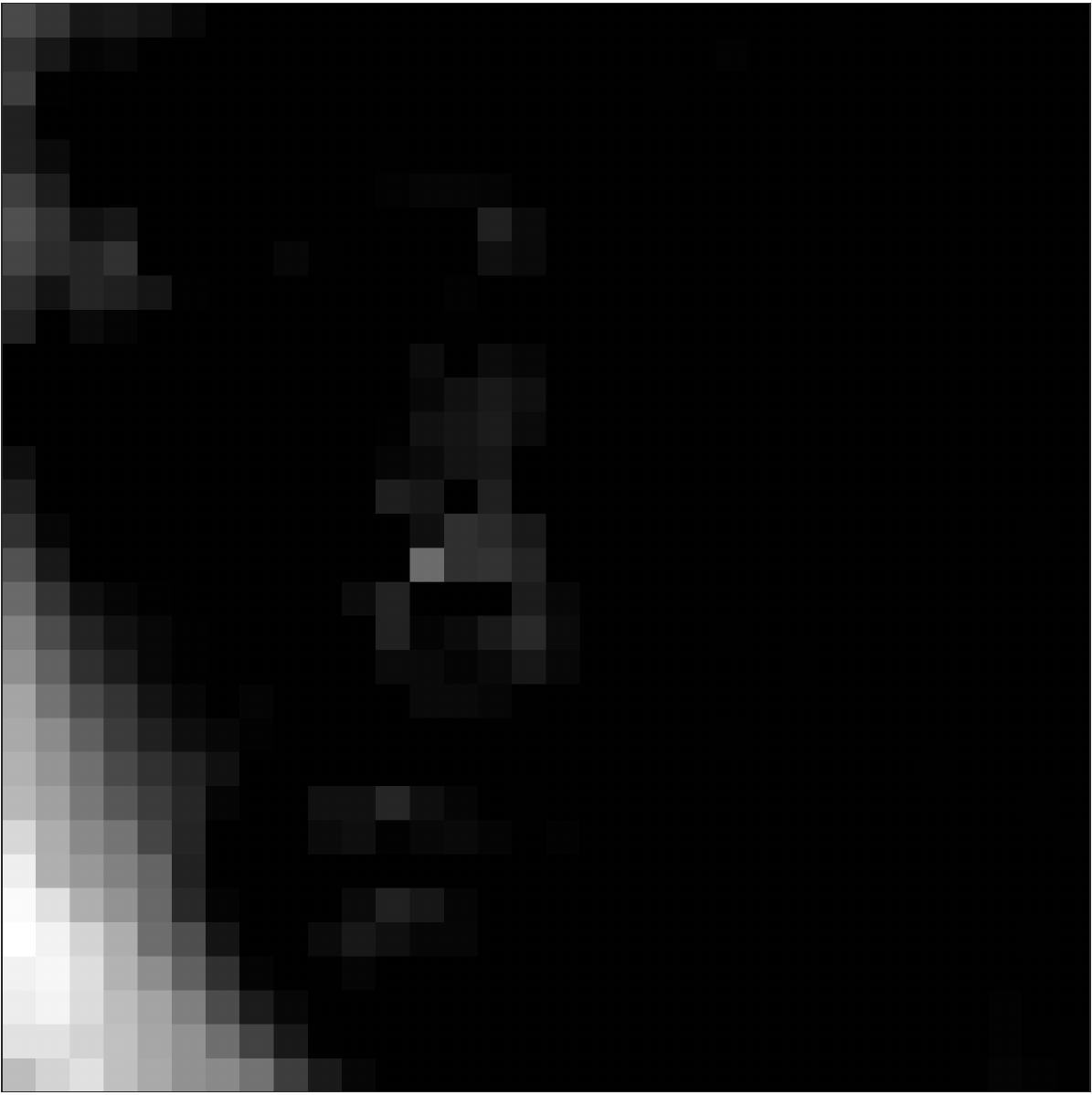} \\
    \end{tabular}
  }
  \caption{(a) Vector $\co$ for the \Eigenfaces data as an image. (b) Matrix $\mAlpha$ for the \Eigenfaces data. (c) Four columns of $\mB$ for the \Eigenfaces data.}
  \label{fig:faces}
\end{figure*}

Similarly to \Worldclim, we can also plot the matrix $\mAlpha$.\!\footnote{Plots of the $\mAlpha$ matrix for the other data sets are in Figure~\ref{fig:param_matrices} in the Appendix.} There we notice that there are some faces that have a strong subtropical structure, and again, most of the structure is mostly NMF.
To validate that also the factors are interpretable, we present examples from the left factor matrix $\mB$ for the \Eigenfaces data in Figure~\ref{fig:faces:B}. We see that factors mostly depict facial features, except the one at the bottom right, which can be used to add lighting effects to the bottom left part of the figures.

\section{Related Work}
\label{sec:related-work}

Nonnegative matrix factorization is a well-studied data analysis method, and over time, many algorithms have been proposed (e.g.~\cite{lee01algorithms,paatero94positive}; see \cite{cichocki09nonnegative} for a comprehensive treatise). NMF has been applied to many data analysis problems (e.g.\ \cite{pauca2004text,xu2003document,brunet2004metagenes}) and algorithms for computing it are included in all major data analysis packages.

Subtropical (or max-times) matrix factorizations are less commonly used in data analysis. The use of approximate low-rank SMF in data analysis was first presented by~\cite{karaev16capricorn} together with the \Capricorn algorithm. \Capricorn is designed for subtropical noise, and later \cite{karaev16cancer} presented the \Cancer algorithm that deals with Gaussian noise. Recently, \Capricorn and \Cancer have been unified under the \Equator framework~\cite{karaev17algorithms_arxiv}, that also provided other quality functions than the squared error.

In general a tropical semiring is any semiring in which the addition operation is max or min. Other than max-times two well studied examples are the max-plus and min-plus semirings \cite{butkovic,heidergrot}. Note that the max-plus and min-plus semirings are isomorphic via the map $h(x)=-x$ and that this transformation preserves the norm $d(x,y)=\abs{x-y}$. Note also that max-plus (tropical) and max-times (subtropical) are also isomorphic via $h(x)=\exp(x)$, but that the norm is not preserved by this transformation~\cite{karaev16capricorn}. This means that whilst the algebraic structures of max-plus and max-times are the same, approximation in max-plus works differently to approximation in max-times. Intuitively max-times gives a lower weight to relative perturbations of smaller numbers. Approximation of network structures by low-rank min-plus matrix factorization is explored in \cite{jlh2017}. Possible applications of max-plus low-rank matrix factorizations to non-linear image processing are discussed in \cite{angulo2017}. 


\section{Conclusions}
\label{sec:conclusions}

Mixed linear--tropical factorization is an interesting novel model for matrix factorization. By smoothly combining factorizations over two algebras, it allows us to model complex structure with an interpretable way. Our algorithm, \Latitude, was able to consistently obtain better reconstruction errors than either NMF or SMF algorithms. Indeed, \Latitude was often better than even SVD. And while SVD comes with well-known limitations to the interpretability, \Latitude's factorization is easier to interpret due to the nonnegative factor matrices and intuitive interpretation of the parameter vectors.

While \Latitude generally showed superior performance compared to NMF of SMF, there were a few instances where it performed slightly worse, which was due to overfitting to the noise. This raises the question of the use of regularization in mixed linear--tropical factorization and is left for future studies.

\Latitude has running time which is linear in the input matrix's dimensions, making it a rather scalable method. Its reliance on nonnegative least-squares optimization, however, can be a limiting factor in scaling \Latitude to big data and distributed systems. Our goal in this paper was to establish the feasibility and usability of mixed linear--tropical models, and developing more scalable algorithms is a natural next step.

In this work we constrained the parameter matrix \mAlpha{} to tropical rank-1 (before the logistic transformation). As we saw in Proposition~\ref{prop:constant_factors}, some constraints are mandatory for sensible decompositions. It is an interesting open question how much more power would a tropical rank-2 parameter matrix give. Also, the relationship between the rank of the factorization and the rank of the parameter matrix is currently unknown. 


\bibliographystyle{abbrv}
\bibliography{library}

\appendix

\section{Varying $k$ Test Without the Subtropical Part}
\label{app:vary_k}

Earlier we have observed that in the varying dimensionality experiments \Latitude and \SVD become very close for higher values of $k$. This inspired a hypothesis that as $k$ grows, the mixed linear-tropical model becomes easier to describe using the standard algebra. Recall that for matrices $\mB \in \rgnk$ and $\mC \in \rgkm$ and parameters $\mAlpha \in [0,1]^{n\times m}$, the element $i, j$ of the mixed linear-tropical matrix product is given by a convex combination of $(\mB \maxprod \mC)_{ij}$ and $(\mB \mC)_{ij}$
\begin{equation} \label{convex:combo}
(\mB \mixprod_{\mAlpha} \mC)_{ij} = \mAlpha_{ij} (\mB \maxprod \mC)_{ij} + (1 - \mAlpha_{ij}) (\mB \mC)_{ij} \;.
\end{equation}

 It is clear that if densities of $\mB$ and $\mC$ remain fixed, then as $k$ grows, the second term of \eqref{convex:combo} becomes more and more dominant. This is because on expectation the sum of $k$ elements grows much faster with $k$ than does their maximum. As a result, when all other parameters are fixed, the influence of the tropical term diminishes as the dimensionality grows, and the data becomes more ``classical''. That does not mean, however, that the structure becomes NMF-like since all the elements inside the NMF part are still scaled by $\mAlpha$. To test our conjecture we once again generated data with varying $k$, but this time without the subtropical term in \eqref{convex:combo}. The results are shown in Figure~\ref{fig:synth:only:nmf}. It is apparent that \SVD has improved compared to normal mixed model, and for $k>8$ it produces better reconstruction errors than \Latitude. It is worth noting that this experiment was made to verify our hypothesis and does not follow the model that \Latitude is designed to solve. As expected, \NMF does not perform well -- scaling by the parameter $\mAlpha$ seems to destroy the structure it is looking for. We did not include \Cancer in this experiment due to its generally poor results for non-tropical data and slow performance for higher $k$s.

\begin{figure*}[t]
  \centering
  \subfloat[Varying rank with mixed data.]{%
    \includegraphics[width=\subfigwidth]{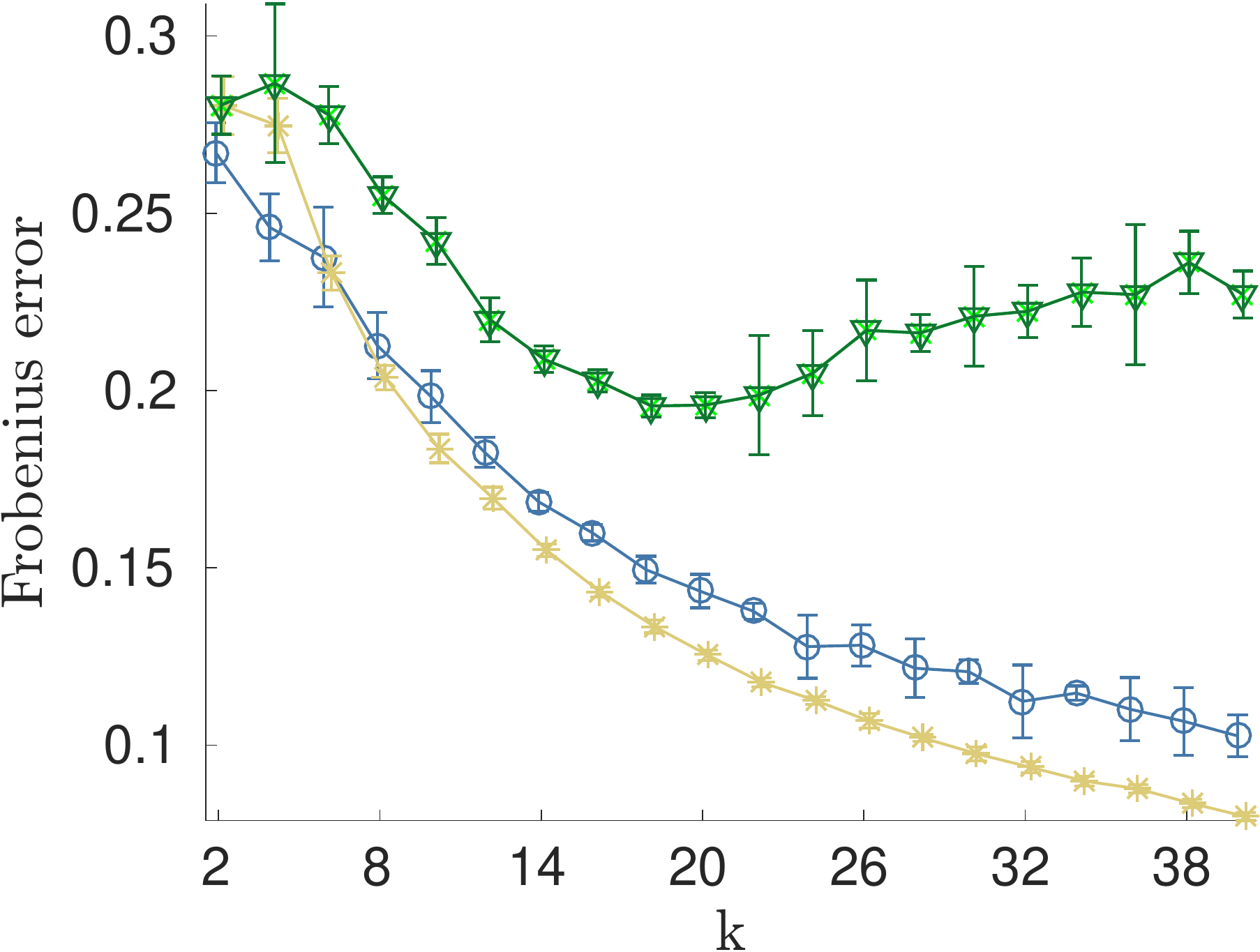}
    \label{rank:only:nmf}
  }
  \caption{\textbf{Reconstruction errors for varying $k$ with the subtropical part of the data removed.} The $x$-axis represents $k$ and the $y$-axis the Frobenius error. All results are averages over 10 random matrices and the width of the error bars is twice the standard deviation.}\label{fig:synth:only:nmf}
\end{figure*} 

\section{Analysis of Real-World Factors}
\label{app:analysis-real-world}

The properties of the real-world data sets are summarized in Table~\ref{tab:real:specs_all}. 

\setlength{\tabcolsep}{0.5em}
   \begin{table}[tb]
   \centering
   \caption{Real world datasets properties.}
   \label{tab:real:specs_all}
   \begin{tabular}{@{}lRRR@{}}
     \toprule
     Dataset & $Rows$ & $Columns$ & $Density$ \\
     \midrule
     \Worldclim  & 2575 & 48  & 99.9\%     \\
     \NPAS         & 1418 & 36  & 99.6\%     \\
     \Eigenfaces & 1024 & 222 & 97.0\%     \\
     \News       & 400  & 800 & 3.5\%      \\
     \HPI            & 253  & 177 & 99.5\%     \\
     \bottomrule
   \end{tabular}
 \end{table}

The factor matrices $\mAlpha$ for \NPAS, \HPI, and \News can be seen in Figure~\ref{fig:param_matrices}.

\begin{figure*}
  \centering
  \subfloat[\NPAS]{%
    \label{fig:wordclim_param:co}%
    \includegraphics[width=\subfigwidth]{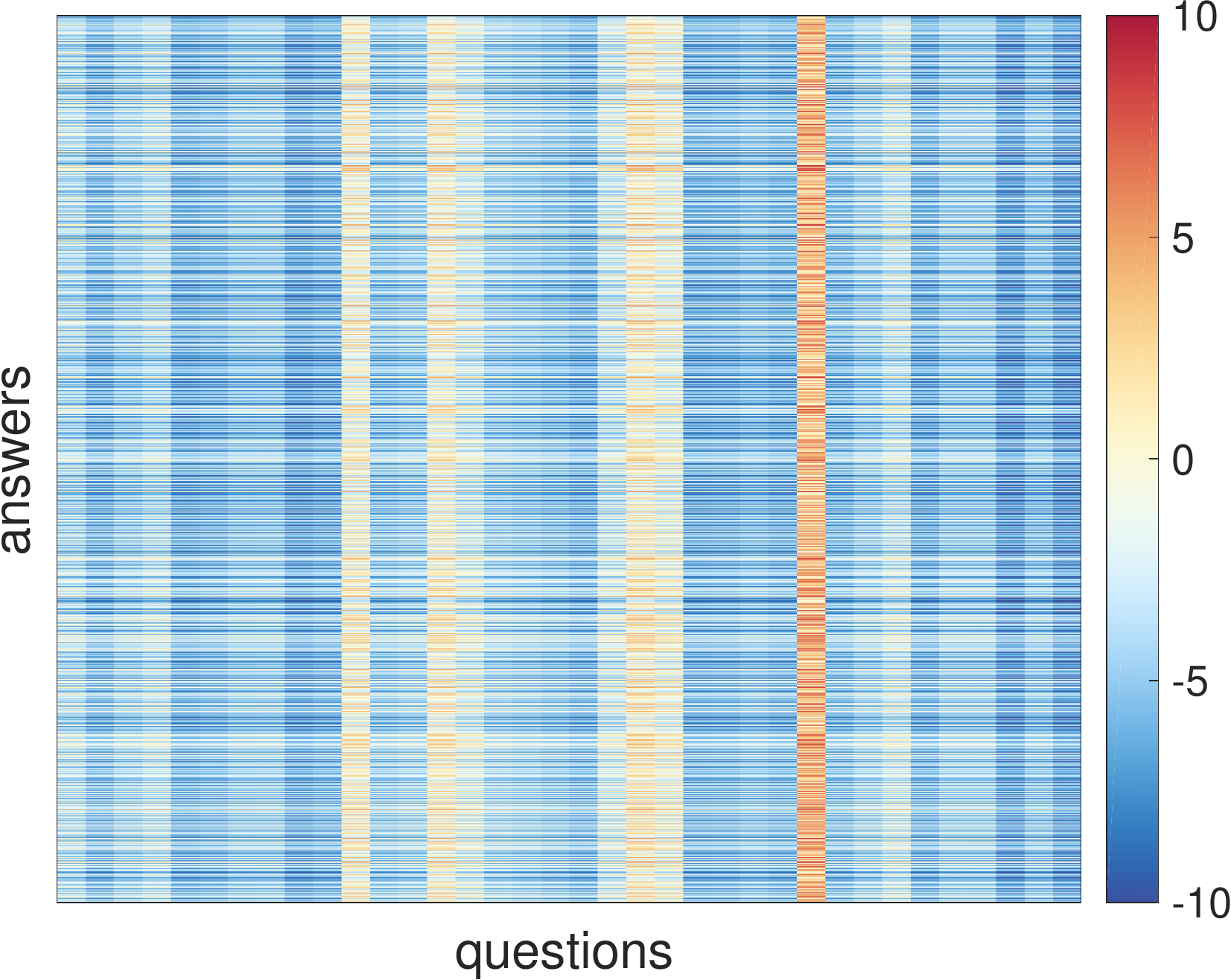}%
  }\qquad
  \subfloat[\HPI]{%
    \label{fig:worldclim_param:ro}%
    \includegraphics[width=\subfigwidth]{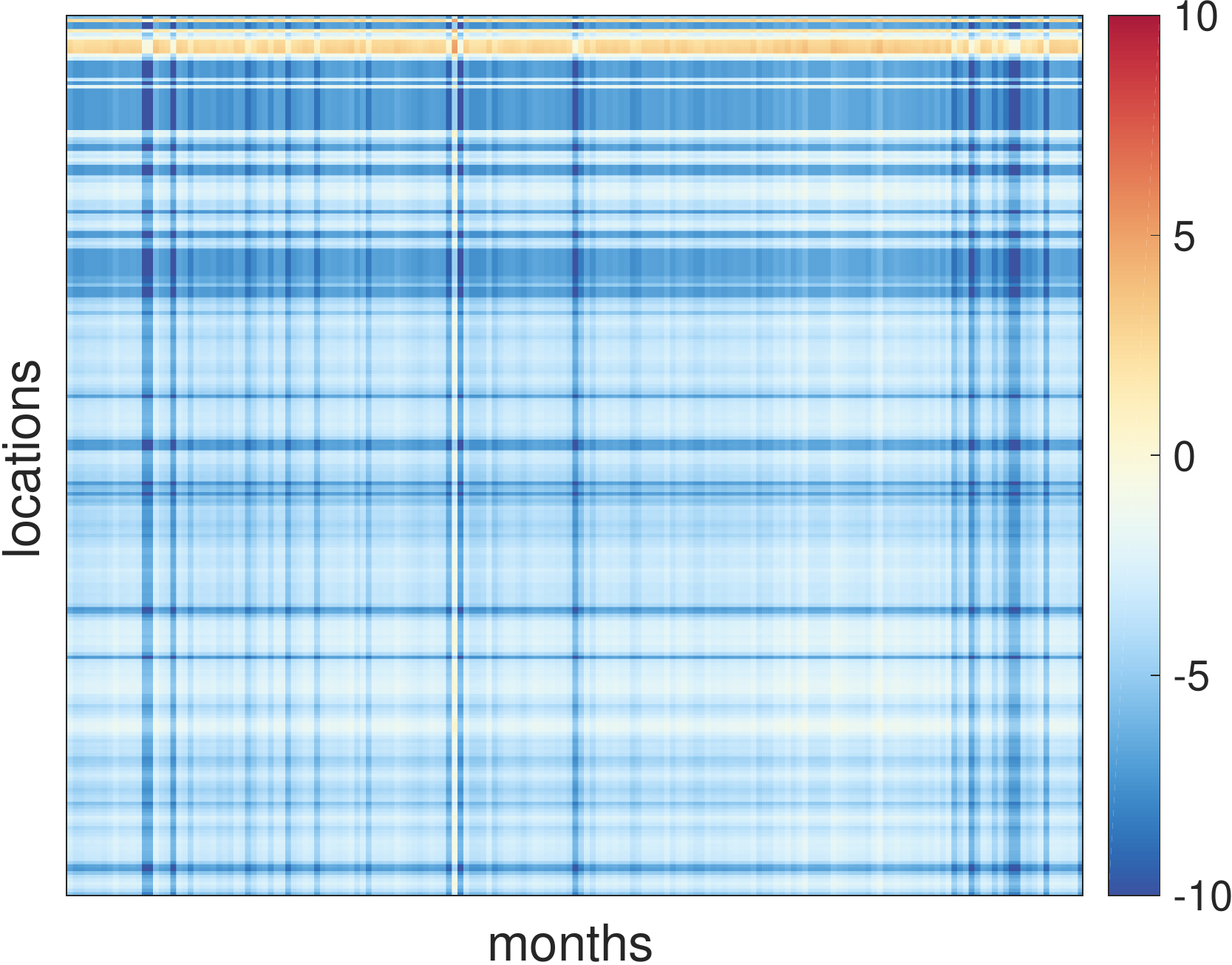}%
  } \qquad
  \subfloat[\News]{%
    \label{fig:worldclim_param:alpha}
    \includegraphics[width=\subfigwidth]{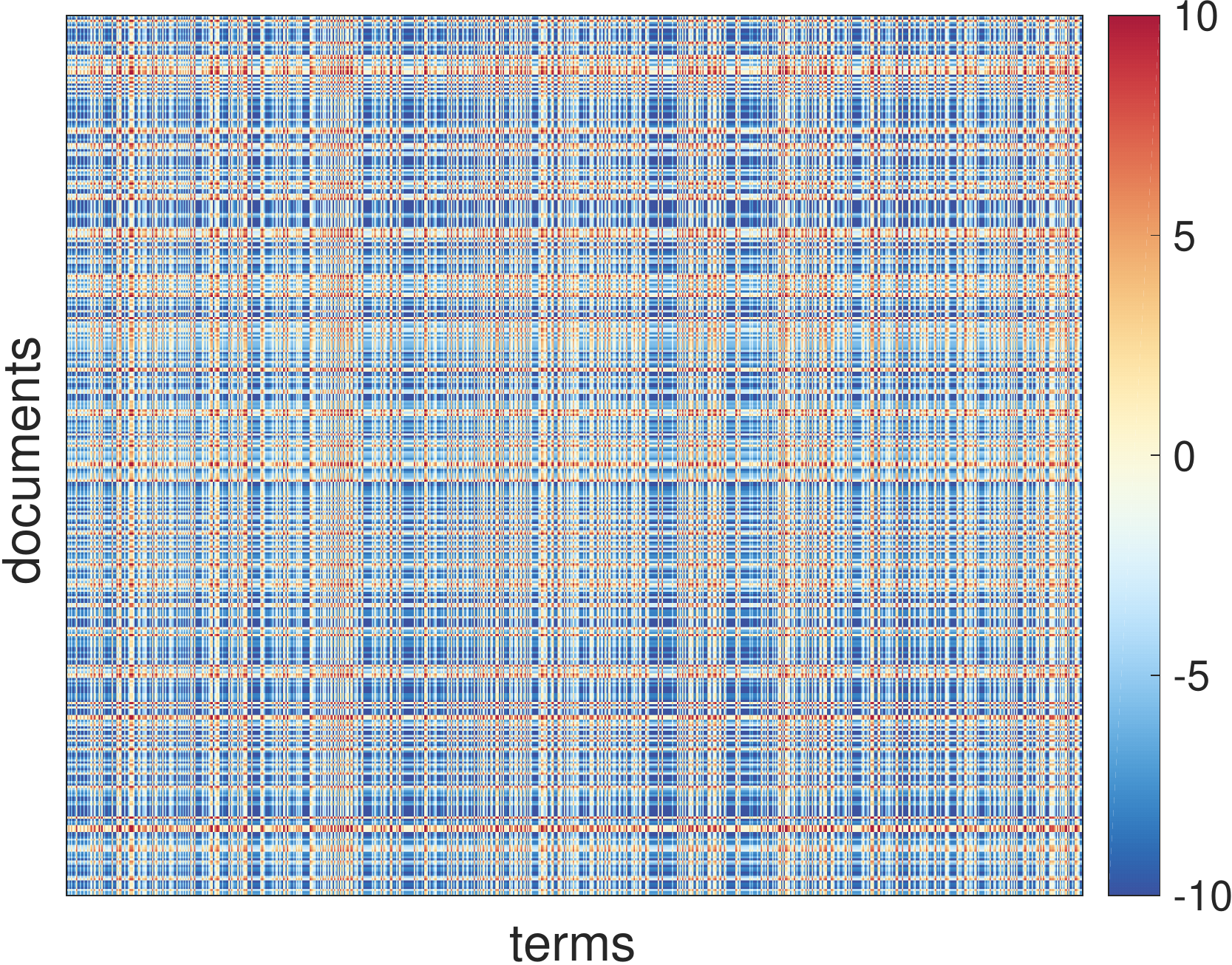}%
  }
  \caption{Visualizations for the parameter matrices with different real-world data sets.}
  \label{fig:param_matrices}
\end{figure*}

\section{Runtime on Real-World Datasets.}
\label{app:runtime}

Table~\ref{tab:real:world:time} shows the execution time of \Latitude and the benchmark algorithms for all datasets used in this paper. Although it is evident that both \SVD and \NMF are much faster, it is worth noting that \Latitude is an iterative algorithm, and its objective improvements tend to become smaller over time. In many cases it produces reasonably high quality results after only a few iterations, which can be used to save execution time. Figure~\ref{fig:error:time} demonstrates the convergence rate of \Latitude.

 \setlength{\tabcolsep}{0.5em}
 \begin{table}[tb] 
   \centering
   \caption{Runtime in seconds for real-world datasets.}
\label{tab:real:world:time}
   \begin{tabular}{@{}lRRRRR@{}}
     \toprule
                      & \text{\Worldclim} & \text{\NPAS} & \text{\Eigenfaces} & \text{\News} & \text{\HPI}  \\
     $k=$         & 10      & 10      & 40      & 20      & 15 \\
     \midrule
     \Latitude  & 60.59 & 30.58 & 148.40 & 52.28 & 10.90     \\
     \Lattrunc  & 57.67 & 28.98 & 143.89 & 49.20 & 11.58     \\
     \SVD        & 1.43 & 0.25 & 0.15 & 0.15 & 0.05     \\
     \NMF       & 1.11 & 0.45 & 2.49 & 1.62 & 0.13     \\
     \Cancer     & 36574 & 11070 & 48476 & 10445 & 785     \\
     \bottomrule
   \end{tabular}
 \end{table}

 \begin{figure*}[t]
   \captionsetup[subfigure]{labelformat=empty}
  \centering
  \subfloat[]{%
    \includegraphics[width=\subfigwidth]{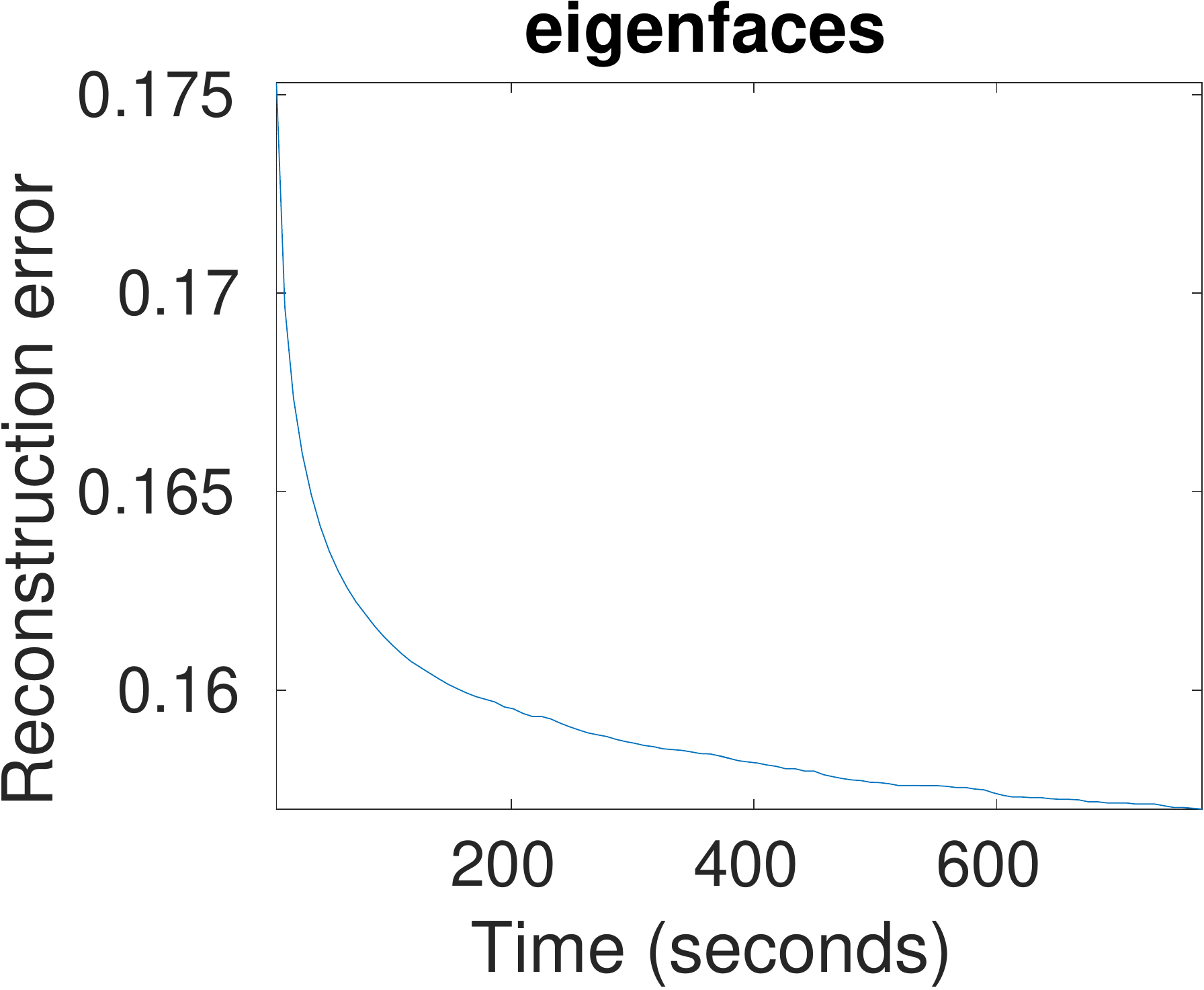}
  }
  \hspace{\subfigspace}
  \subfloat[]{%
    \includegraphics[width=\subfigwidth]{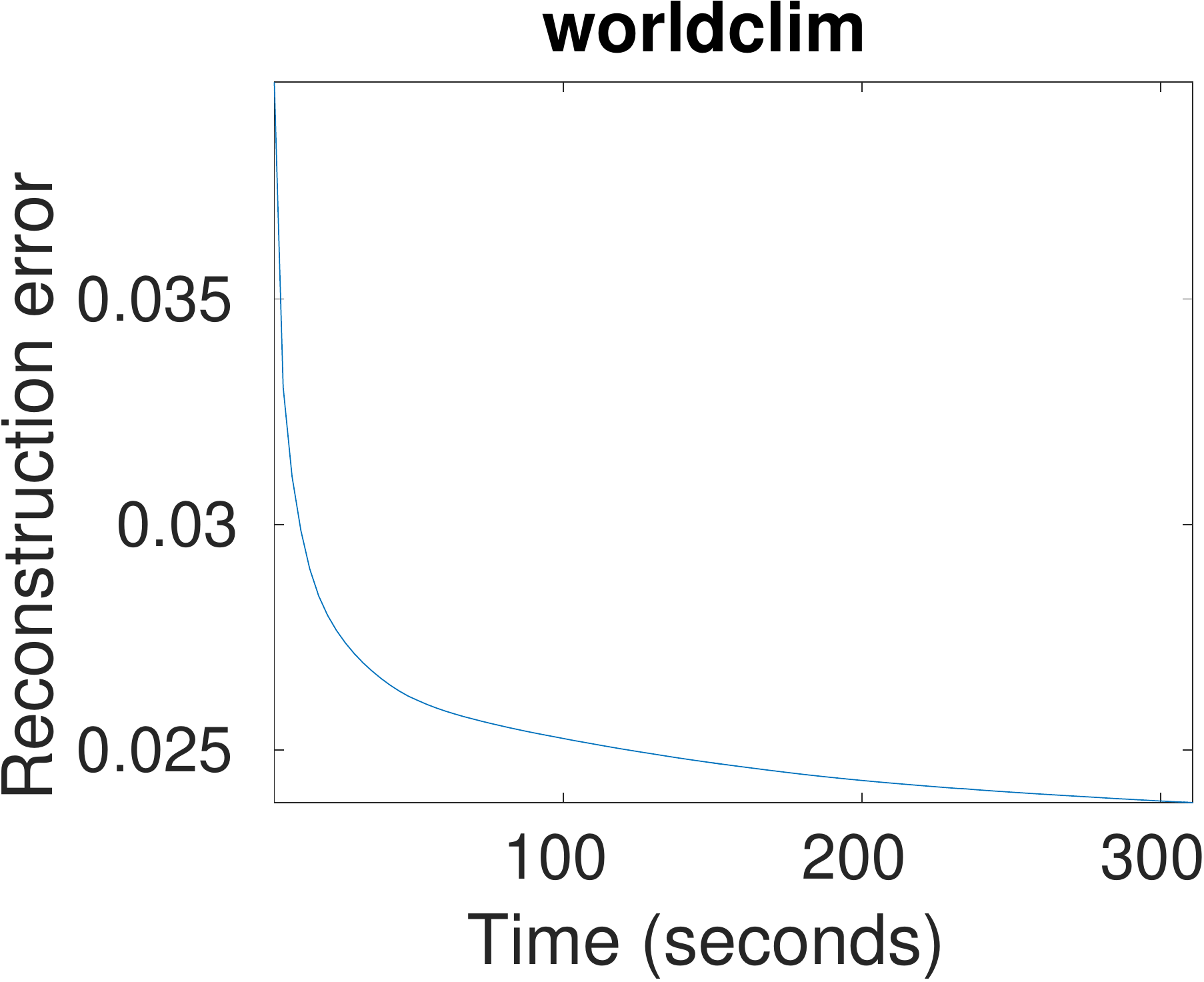}
  }
  \hspace{\subfigspace}
  \subfloat[]{%
    \includegraphics[width=\subfigwidth]{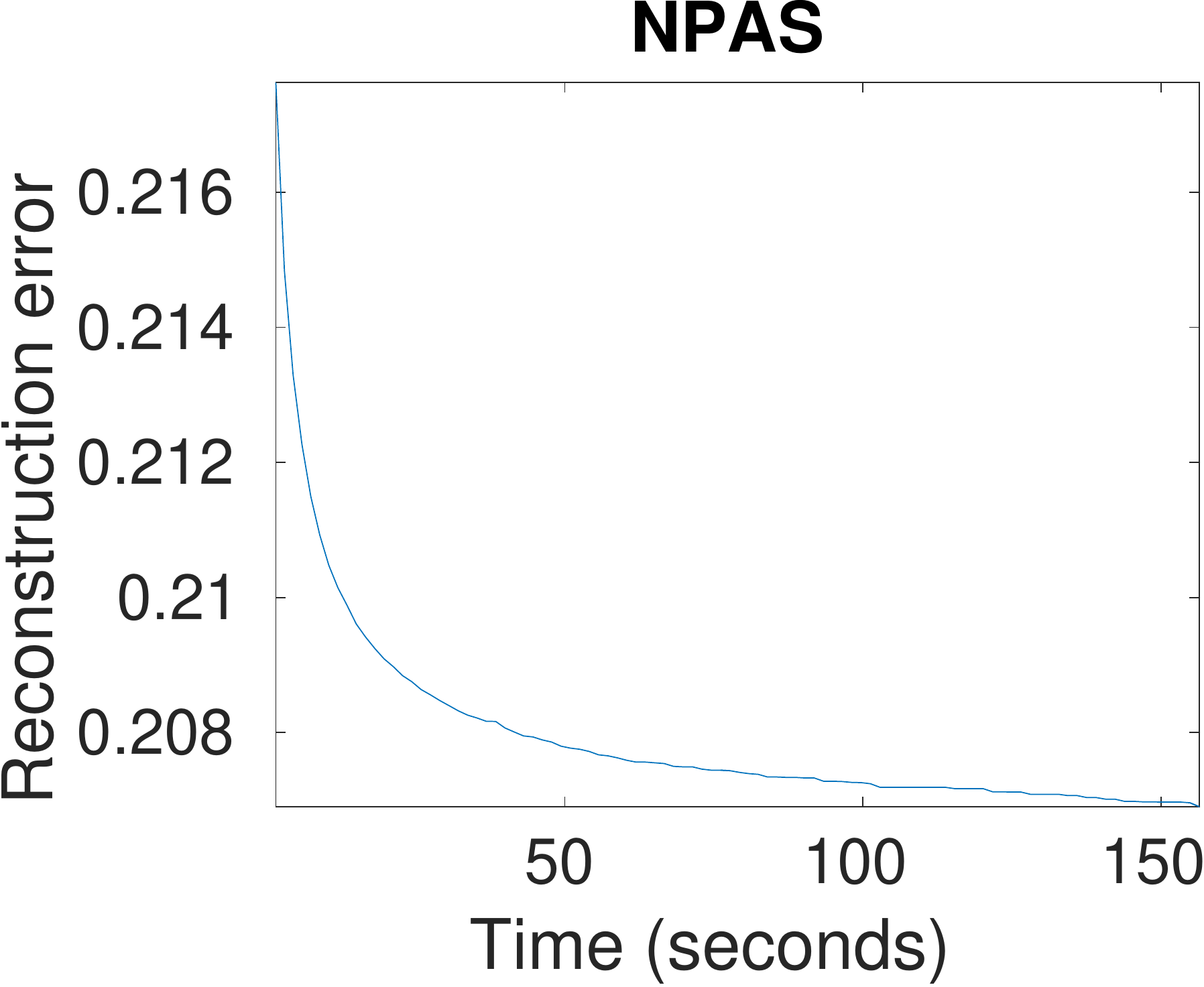}
  }
  \hspace{\subfigspace}
  \subfloat[]{%
    \includegraphics[width=\subfigwidth]{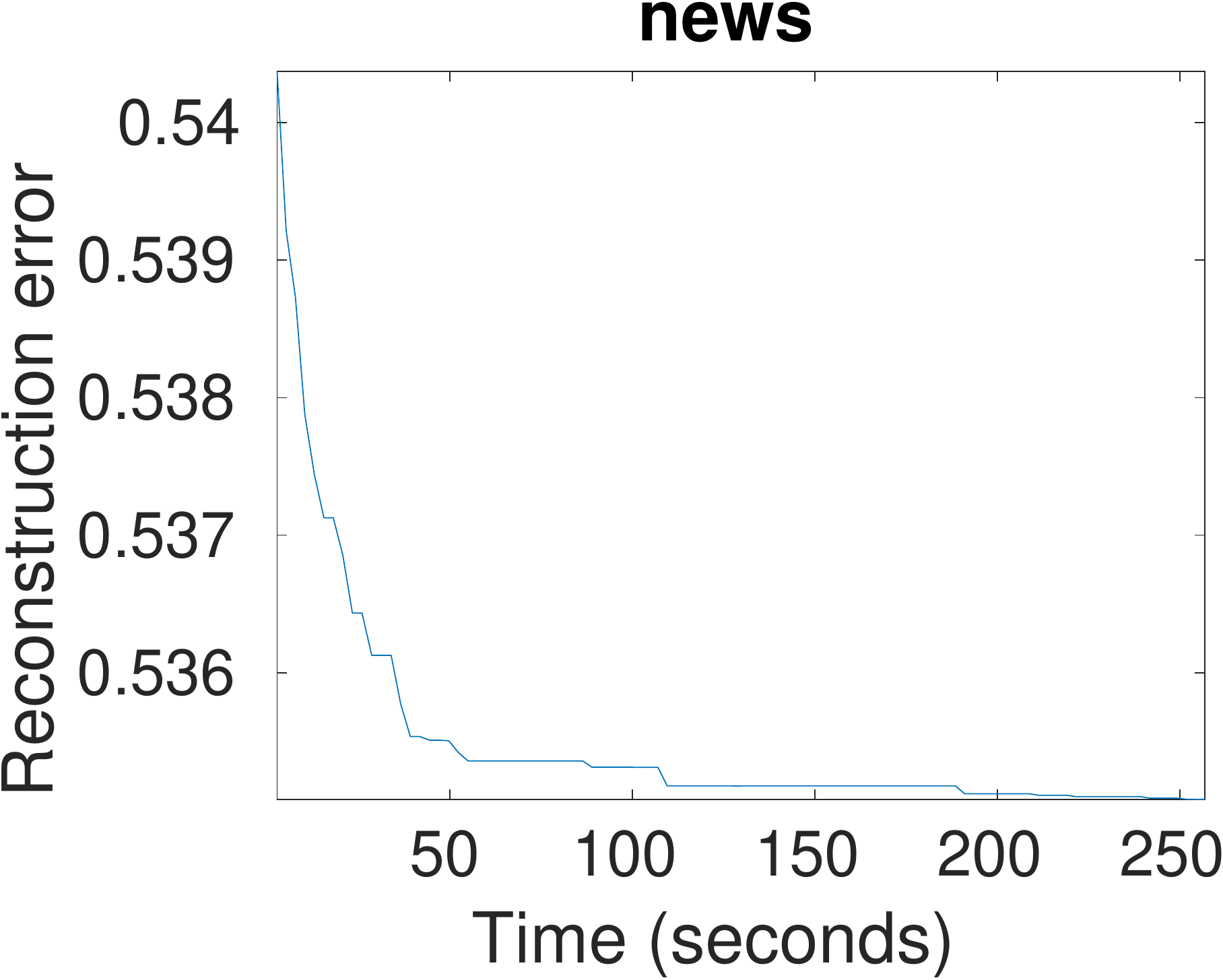}
  }
  \subfloat[]{%
    \includegraphics[width=\subfigwidth]{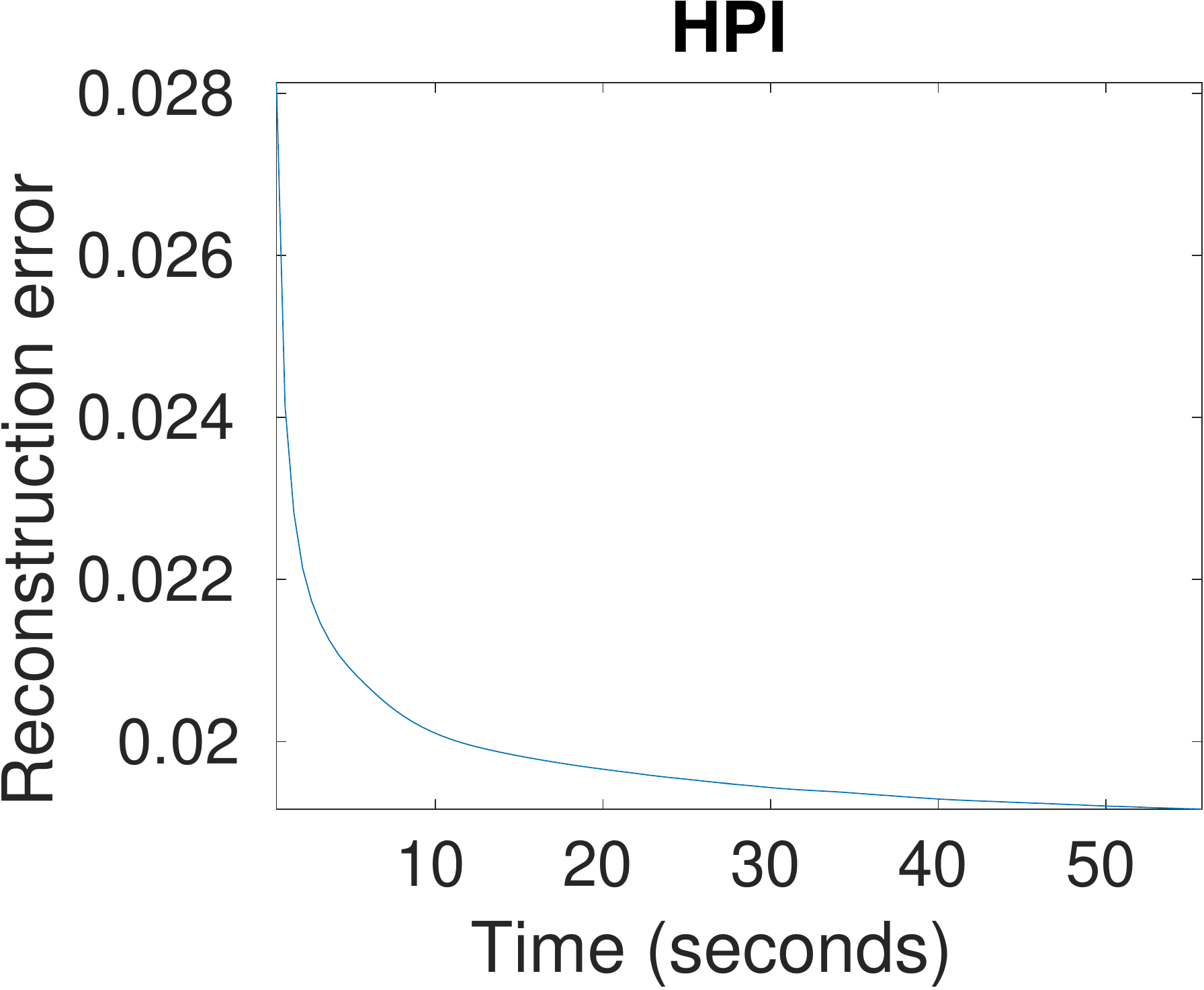}
  }
  \caption{\textbf{Reconstruction error of \Latitude as a function of time for real-world datasets.} \label{fig:error:time}}
\end{figure*}

\end{document}